\documentclass[11pt]{article}

\usepackage[margin=1in]{geometry}
\usepackage{natbib}
\usepackage{hyperref}
\usepackage{url}
\usepackage{amsthm}
\usepackage{aliascnt}
\usepackage{float}
\usepackage{booktabs}

\usepackage{amsmath,amsfonts,bm}
\usepackage{graphicx}
\usepackage{algorithm}
\usepackage{algpseudocode}
\usepackage{comment}

\def\eqref#1{equation~\ref{#1}}
\def\Eqref#1{Equation~\ref{#1}}

\def\1{\bm{1}}

\DeclareMathAlphabet{\mathsfit}{\encodingdefault}{\sfdefault}{m}{sl}
\SetMathAlphabet{\mathsfit}{bold}{\encodingdefault}{\sfdefault}{bx}{n}

\newcommand{\E}{\mathbb{E}}

\newcommand{\R}{\mathbb{R}}

\DeclareMathOperator*{\argmin}{arg\,min}

\hypersetup{hidelinks}

\newtheorem{theorem}{Theorem}
\newaliascnt{proposition}{theorem}
\newtheorem{proposition}[proposition]{Proposition}
\aliascntresetthe{proposition}
\newaliascnt{lemma}{theorem}
\newtheorem{lemma}[lemma]{Lemma}
\aliascntresetthe{lemma}
\newaliascnt{corollary}{theorem}
\newtheorem{corollary}[corollary]{Corollary}
\aliascntresetthe{corollary}

\newtheorem{definition}{Definition}
\newtheorem{assumption}{Assumption}

\usepackage{cleveref}

\crefname{algorithm}{Algorithm}{Algorithms}
\Crefname{algorithm}{Algorithm}{Algorithms}
\crefname{assumption}{Assumption}{Assumptions}
\Crefname{assumption}{Assumption}{Assumptions}
\crefname{definition}{Definition}{Definitions}
\Crefname{definition}{Definition}{Definitions}
\crefname{theorem}{Theorem}{Theorems}
\Crefname{theorem}{Theorem}{Theorems}
\crefname{proposition}{Proposition}{Propositions}
\Crefname{proposition}{Proposition}{Propositions}
\crefname{lemma}{Lemma}{Lemmas}
\Crefname{lemma}{Lemma}{Lemmas}
\crefname{corollary}{Corollary}{Corollaries}
\Crefname{corollary}{Corollary}{Corollaries}
\crefname{remark}{Remark}{Remarks}
\Crefname{remark}{Remark}{Remarks}

\title{When Is Coarse Supervision Worth It? Cost-Aware Learning under Unknown Aggregation}
\author{%
Jianyu Xu$^{1,2}$ \quad Smriti Jha$^{2}$ \quad Aarti Singh$^{2}$ \quad Bryan Wilder$^{2}$\\[0.6em]
\small $^{1}$University of North Carolina at Charlotte, Charlotte, NC 28223\\
\small $^{2}$Carnegie Mellon University, Pittsburgh, PA 15213
}
\date{}

\begin{document}
\maketitle

\begin{abstract}
Modern learning systems often acquire supervision at multiple resolutions, trading annotation cost against information content.
We study cost-aware two-resolution learning, where expensive fine labels reveal a vector response and cheaper coarse labels reveal a scalar aggregate formed with unknown weights, while the target remains the full response.
The challenge is that unknown aggregation changes which directions coarse data can identify, so the value of coarse supervision depends jointly on cost, noise, and identification.
We characterize this information geometry and develop an estimate-and-track policy that learns the aggregation rule and tracks the optimal resolution mix.
We derive a closed-form break-even condition for coarse supervision and prove that the online policy attains the optimal leading cumulative-risk coefficient, with a matching local asymptotic minimax lower bound.
Synthetic experiments support the predicted all-fine/mixed transition, show the online learner approaching the oracle-share benchmark, and demonstrate a finite-budget gain over all-fine acquisition when coarse supervision is sufficiently favorable.
Our results provide a principled way to balance information and annotation cost across supervision resolutions.

\end{abstract}

\section{Introduction}
\label{sec:introduction}
Modern learning systems increasingly acquire supervision at different resolutions. For the same example one may purchase a detailed record of many attributes or a cheaper summary of them: pixel-level versus image-level annotation in vision, token-level versus sentence-level labels in language, a detailed diagnostic profile versus an overall outcome, or a multi-criterion evaluation rubric versus one overall score. A growing body of work in active learning and weak supervision accordingly treats the \emph{level} of supervision, and not only the choice of example, as part of the acquisition decision \citep{tejero2023full,hu2018active,rotman2022multi,matsuo2025instance,nguyen2026leveraging}. This raises a basic question: when supervision is available at several resolutions and prices, how should a finite annotation budget be divided across them?

Consider evaluating model-generated responses. A detailed annotation scores a response on several criteria such as factuality, relevance, safety, clarity, and style, while a cheaper annotation asks only for an overall judgment. The quantity we wish to predict may nevertheless be the full criterion vector. The cheap label is therefore not simply a noisier copy of the target: it is a scalar projection of a $K$-dimensional response, carrying information only along one direction. The relevant tradeoff is whether the information gained from this inexpensive direction compensates for the fine information forgone elsewhere.

A further difficulty is that the aggregation rule is often not known to the learner. An overall score may combine fine criteria according to weights implicit in an evaluation protocol or annotator population rather than specified in advance. If the rule were known, the cheap channel would supply information along a known response direction. When it is unknown, that direction must itself be identified from fine and coarse observations before that information can be fully exploited. Cheap supervision therefore creates an additional identification cost: its value depends not only on price and noise, but also on uncertainty about what is being aggregated.

We study this tradeoff in a stylized linear two-resolution model. Each example carries a $d$-dimensional covariate; a fine query returns a noisy $K$-dimensional response generated by a linear map, and a cheaper but coarse query returns a noisy scalar aggregation under unknown simplex weights. The inferential target is the entire fine map, measured in prediction risk. Linearity makes the nuisance-efficient information geometry exactly computable, yielding closed-form break-even conditions, optimal resolution mixes, and leading-order annotation savings.

Our approach starts from a simple geometric question: which parts of the fine map can each annotation type distinguish? If the aggregation rule were known, coarse labels would inform a fixed response direction, while fine labels identify the remaining directions. When the weights are unknown, however, some changes in the fine map can be matched by changes in the aggregation weights and therefore leave the coarse conditional mean unchanged. In the linear model, this ambiguity removes exactly $K-1$ directions from the part of the target that coarse data can help estimate. The remaining information still separates into fine-only and shared blocks, so the design problem reduces to their dimensions and a single scalar measuring cost-adjusted precision. This yields a one-dimensional objective and a sharp break-even boundary between all-fine and mixed acquisition.

Our problem connects multiresponse optimal design, multifidelity estimation, and cost-sensitive supervision \citep{sagnol2011computing, peherstorfer2016optimal, xu2022bandit, matsuo2025instance, fotakis2021efficient}. These literatures separately address heterogeneous observation operators, budget allocation, and supervision granularity. What is specific here is that the cheap channel is a lower-dimensional projection of the prediction target whose direction is itself unknown, and that we seek exact first-order answers for the resulting cost--information tradeoff.  

\textbf{Summary of Contributions.}
Under the canonical linear model, we develop a first-order theory of cost-aware two-resolution supervision with unknown aggregation. Specifically, we separate two sources of difficulty. First, not knowing the aggregation direction changes the information available from coarse labels and creates an intrinsic statistical price. Second, the optimal resolution mix depends on that unknown direction, so the learner must estimate and track the allocation online. We show that this second difficulty need not add another first-order penalty.

\begin{itemize}
    \item \textbf{A sharp static cost--information frontier.} For fixed $d$ and $K$, a single effective ratio $\lambda$, combining relative annotation costs, noise variances, and the aggregation strength $\|w_\star\|_2^2$, determines whether coarse supervision is worth purchasing. With unknown simplex aggregation weights, coarse labels enter the optimal design if and only if $\lambda>dK/(d-K+1)$, compared with threshold $K$ when the aggregation rule is known. Above the boundary, both the optimal coarse spending share and its leading fractional gain are available in closed form. The gain measures the leading reduction in prediction risk at fixed budget, equivalently the leading reduction in annotation cost at fixed target risk, and remains structurally bounded even when coarse feedback becomes arbitrarily cheap or precise.
    \item \textbf{First-order optimal online adaptation with a matching lower bound.} When the aggregation rule is unknown, so is the optimal resolution mix. We give an epochal estimate-and-track policy that reserves a vanishing design stream to estimate the aggregation rule and target share, while the remaining budget tracks that share. A cross-fitted estimator uses the estimation stream without feeding its labels back into future allocation decisions. The cumulative risk matches the optimal leading $\log T$ coefficient for exogenous static designs under unknown aggregation, with an $O(1)$ remainder. A matching local asymptotic minimax lower bound shows that learning the allocation adds no first-order penalty beyond the intrinsic nuisance price.
\end{itemize}

Synthetic experiments support both parts of the theory. A structural comparison illustrates the gap between known- and unknown-aggregation acquisition decisions, while the online policy approaches the oracle-share benchmark and, in a strongly favorable coarse-information regime, yields a clear finite-budget improvement over all-fine acquisition.

\section{Related Works}
\label{sec:related_works}
We discuss the closest connections to optimal design, multifidelity estimation, and cost-sensitive learning. Additional comparisons with the broader literature appear in Appendix~\ref{app:extended-related}.

\textbf{Optimal design and cost-aware partial measurement.}
Multiresponse optimal design allows experiment types to contribute different vector observations and information matrices. \citet{sagnol2011computing} develops A-optimal design under linear resource constraints, while later work treats multivariate responses and variable response dimensions \citep{soumaya2015optimal,somogyi2026randomized,filova2026optimal}. With known aggregation, our fine and coarse channels are two structured observation types in this framework; related work also optimizes complete versus partial measurement under collection costs \citep{mareis2026cost}. Our focus is the closed-form two-resolution frontier and how it changes when the coarse observation operator is itself unknown.

\textbf{Multifidelity estimation and adaptive design.}
Multifidelity methods allocate budgets across expensive and cheap information sources, including analytic allocations and multi-output estimators \citep{peherstorfer2016optimal,schaden2020multilevel,croci2023multi}. Cross-fidelity relationships can also be learned adaptively \citep{xu2022bandit,dixon2026optimally}, while \citet{fontaine2021online} study online A-optimal scalar regression under unknown noise. These methods motivate our adaptive question, but our cheap channel is a lower-dimensional projection of the same vector target, and the unknown projection is a nuisance parameter in the conditional-mean law itself.

\textbf{Cost-sensitive learning and supervision granularity.}
Active-learning work has long selected among annotation sources or forms with different costs and information content \citep{donmez2008proactive,wallace2011should,hu2018active,rotman2022multi,tejero2023full,matsuo2025instance}. Related work studies learning from intrinsically coarse labels or cheap coarse machine annotations \citep{fotakis2021efficient,nguyen2026leveraging}. These settings use different observation structures; here the cheaper action is an unknown scalar projection of the same vector-valued response whose full conditional mean remains the target.

\section{Problem Setup}
\label{sec:problem_setup}
We formulate the two-resolution acquisition problem, define the performance criterion and reference regimes, and state the regularity assumptions used throughout.

\subsection{Problem Formulation}
\label{subsec:problem-formulation}

We study a sequential information-acquisition problem with two annotation resolutions. At round $t$, the learner chooses the resolution $A_t\in\{\mathrm F,\mathrm C\}$ as a measurable function of the history before observing the current covariate $x_t\in\R^d$ or current response noise. The prediction target is the full $K$-dimensional conditional response $z_\star(x)=\Theta_\star^\top x$, where $\Theta_\star\in\R^{d\times K}$ is unknown.

A fine query returns
\begin{equation}
Y_t^{\mathrm F}
=
\Theta_\star^\top x_t+\varepsilon_t^{\mathrm F},
\qquad
\varepsilon_t^{\mathrm F}\sim\mathcal N(0,\sigma_{\mathrm F}^2I_K),
\label{eq:fine-observation}
\end{equation}
and costs $c_{\mathrm F}>0$. A coarse query costs $c_{\mathrm C}>0$ and returns only
\begin{equation}
Y_t^{\mathrm C}
=
w_\star^\top\Theta_\star^\top x_t+\varepsilon_t^{\mathrm C}
=
x_t^\top u_\star+\varepsilon_t^{\mathrm C},
\qquad
u_\star:=\Theta_\star w_\star,
\qquad
\varepsilon_t^{\mathrm C}\sim\mathcal N(0,\sigma_{\mathrm C}^2),
\label{eq:coarse-observation}
\end{equation}
where $w_\star$ lies in the probability simplex $\Delta_K:=\{w\in\R^K:\mathbf 1^\top w=1,\;w_k\ge0\}$. Hence coarse supervision reveals a scalar projection of the same fine conditional mean rather than a separate prediction target.

The round index counts completed queries, while the budget clock records their cumulative annotation cost. For a budget level $b$, let $N_{\mathrm F}(b)$ and $N_{\mathrm C}(b)$ be the numbers of completed queries whose cumulative annotation cost does not exceed $b$. Bounded action costs can leave a bounded unused residue, so $c_{\mathrm F}N_{\mathrm F}(b)+c_{\mathrm C}N_{\mathrm C}(b)\le b$. The learner may use the entire observed history to choose future resolutions but never observes the unqueried response. We consider both known and unknown aggregation. Known $w_\star$ provides a static reference problem; our main setting treats $w_\star$ as unknown and learns the allocation sequentially.

\subsection{Key Quantities}
\label{subsec:key-quantities}

\begin{definition}[Fine-prediction risk]
\label{def:prediction-risk}
For an estimator $\widehat\Theta$, define $\ell(\widehat\Theta,\Theta_\star):=\E_x[\|\widehat\Theta^\top x-\Theta_\star^\top x\|_2^2]$. Let $B_0$ denote the end of a fixed finite initialization and the start of the reported budget clock. For budget $b\ge B_0$, let $L_b:=\E[\ell(\widehat\Theta_b,\Theta_\star)]$, with the current estimator held fixed between completed estimation queries. The cumulative fine-prediction risk through budget $T$ is
\begin{equation}
\mathfrak R_T
:=
\sum_{b=B_0}^{T}L_b.
\label{eq:cumulative-risk}
\end{equation}
\end{definition}

Therefore, $T$ is an annotation-cost horizon rather than a query count. The budget clock in \Cref{eq:cumulative-risk} uses integer cost units for convenience; rational costs can be rescaled, and a continuous budget integral gives the same first-order conclusions. This criterion measures the anytime quality of a predictor updated throughout an annotation campaign, rather than only its terminal risk.

\begin{definition}[Reference regimes]
\label{def:benchmarks}
We use three reference regimes.
\begin{itemize}
\setlength{\itemsep}{0pt}
\setlength{\parskip}{0pt}
\setlength{\parsep}{0pt}
\setlength{\topsep}{0pt}
    \item[$\mathsf B_1$:] \emph{known-$w_\star$ static oracle}: the best exogenous static allocation when $w_\star$ is known;
    \item[$\mathsf B_2$:] \emph{unknown-aggregation oracle-share benchmark}: the true-parameter optimal share is given exogenously, but $\Theta_\star$ and $w_\star$ remain unknown and the schedule carries no inferential information;
    \item[$\mathsf B_3$:] \emph{online learned-allocation setting}: $w_\star$ and the optimal share are learned sequentially; \Cref{thm:online-main} addresses it.
\end{itemize}
\end{definition}

\subsection{Assumptions}
\label{subsec:assumptions}

We first impose standard regularity on the covariates.

\begin{assumption}[Covariates]
\label{ass:covariates}
The covariates are i.i.d. with $\E[x_t]=0$ and covariance $\Sigma_x\succ0$. Unlimited unlabeled covariates are available at zero annotation cost, so $\Sigma_x$ is treated as known. After whitening, $\E[x_tx_t^\top]=I_d$ and $\|x_t\|_2\le L_x$ almost surely.
\end{assumption}

Whitening only changes coordinates and is without loss for prediction risk, and bounded covariates are a standard regularity condition.

\begin{assumption}[Noise and costs]
\label{ass:noise-cost}
Fine-query noises are i.i.d. $\mathcal N(0,\sigma_{\mathrm F}^2I_K)$ and coarse-query noises are i.i.d. $\mathcal N(0,\sigma_{\mathrm C}^2)$. The queried noises are independent across items, independent of covariates, and independent across resolutions conditional on the covariate. The positive costs $c_{\mathrm F},c_{\mathrm C}$ and noise variances $\sigma_{\mathrm F}^2,\sigma_{\mathrm C}^2$ are known.
\end{assumption}

The Gaussian independent-channel model is the canonical setting in which the two resolutions have separate measurement noise; correlated or shared item-level noise would lead to a different pairing problem (see \Cref{app:limitations}). We next require the aggregation parameter to be regular and identifiable.

\begin{assumption}[Aggregation regularity and identifiability]
\label{ass:identifiability}
The dimensions satisfy $d\ge K\ge2$. For some fixed $\tau\in(0,1/(2K))$, $w_\star$ lies in the simplex interior with $\min_k w_{\star,k}\ge2\tau$. Let $Q\in\R^{K\times(K-1)}$ have orthonormal columns spanning $\mathbf 1^\perp$. For fixed constants $\kappa_\Theta,M_\Theta>0$, the whitened coefficient matrix satisfies
\begin{equation}
\sigma_{\min}(\Theta_\star Q)\ge\kappa_\Theta,
\qquad
\|\Theta_\star\|_{\mathrm{op}}\le M_\Theta.
\label{eq:tangent-identifiability}
\end{equation}
\end{assumption}

The interior condition fixes the local simplex dimension at $K-1$, while the singular-value condition rules out observationally indistinguishable aggregation directions. Finally, we specify the sequential policy class and asymptotic regime.

\begin{assumption}[Sequential protocol and asymptotics]
\label{ass:sequential}
Each action $A_t$ is measurable with respect to the history before round $t$ and is therefore conditionally independent of the current $x_t$ and current noises. The dimensions, costs, noise variances, and regularity constants are fixed as $T\to\infty$.
\end{assumption}

Choosing the resolution before seeing the current covariate isolates the resolution-allocation problem; allowing current-item leverage would introduce an additional optimal-design dimension.

\section{Static Allocation with Known Aggregation}
\label{sec:static}
We first isolate the cost--information tradeoff when the aggregation vector $w_\star$ is known. This benchmark removes the need to learn the coarse direction and reveals, in closed form, when coarse supervision is worth purchasing, how much budget should be assigned to it, and how much it can improve over all-fine acquisition.

\subsection{Fine-Prediction Risk under a Static Allocation}
\label{subsec:static-risk}

Fix fine and coarse sample counts $N_{\mathrm F}$ and $N_{\mathrm C}$. Let $\beta=\operatorname{vec}(\Theta)$ and define $\alpha:=N_{\mathrm F}/\sigma_{\mathrm F}^2$ and $\gamma:=N_{\mathrm C}/\sigma_{\mathrm C}^2$. Since $w_\star$ is known, the expected Fisher information for $\beta$ is
\begin{equation}
\mathcal I_{\mathrm{kn}}
=
\alpha I_{dK}
+
\gamma\bigl(w_\star w_\star^\top\otimes I_d\bigr).
\label{eq:known-information}
\end{equation}
Fine labels contribute in every response direction, while the coarse term acts only along $w_\star$. 

\begin{proposition}[Known-aggregation information trace]
\label{prop:known-static-risk}
Under \Cref{ass:covariates,ass:noise-cost}, $\mathcal I_{\mathrm{kn}}$ has eigenvalue $\alpha$ with multiplicity $d(K-1)$ and eigenvalue $\alpha+\gamma\|w_\star\|_2^2$ with multiplicity $d$. Hence
\begin{equation}
\operatorname{tr}(\mathcal I_{\mathrm{kn}}^{-1})
=
d(K-1)\frac{\sigma_{\mathrm F}^2}{N_{\mathrm F}}
+
d\left(
\frac{N_{\mathrm F}}{\sigma_{\mathrm F}^2}
+
\frac{N_{\mathrm C}\|w_\star\|_2^2}{\sigma_{\mathrm C}^2}
\right)^{-1}.
\label{eq:known-static-risk}
\end{equation}
For exogenous static allocations, this inverse-information trace is the first-order local minimax fine-prediction risk and is attained by the Gaussian efficient estimator described in Appendix~\ref{app:proof-known-risk}.
\end{proposition}

The $d(K-1)$ directions orthogonal to $w_\star$ can only be learned from fine labels, whereas the remaining $d$ directions can use both resolutions. This separation makes the cost-allocation problem one-dimensional.

\subsection{Cost-Optimal Resolution Mix}
\label{subsec:static-allocation}

Let $\eta\in[0,1)$ denote the fraction of annotation budget spent on coarse labels. At total budget $B$, ignoring bounded integer effects, $N_{\mathrm F}=(1-\eta)B/c_{\mathrm F}$ and $N_{\mathrm C}=\eta B/c_{\mathrm C}$. Substituting into \Cref{eq:known-static-risk} introduces the effective coarse-information ratio
\begin{equation}
\lambda
:=
\frac{c_{\mathrm F}\sigma_{\mathrm F}^2\|w_\star\|_2^2}
{c_{\mathrm C}\sigma_{\mathrm C}^2},
\label{eq:known-lambda}
\end{equation}
which combines relative annotation cost, relative noise, and aggregation strength. For simplex weights, $\|w_\star\|_2^2\in[1/K,1]$: concentrated aggregation makes each coarse label more informative in this sense, while uniform averaging is least informative. The resulting leading risk is
\begin{equation}
\mathcal R_{\mathrm K}(B,\eta)
=
\frac{c_{\mathrm F}\sigma_{\mathrm F}^2d}{B}
\psi_{\mathrm K}(\eta,\lambda),
\qquad
\psi_{\mathrm K}(\eta,\lambda)
:=
\frac{K-1}{1-\eta}
+
\frac{1}{1+(\lambda-1)\eta}.
\label{eq:known-budget-risk}
\end{equation}

\begin{theorem}[Static break-even boundary]
\label{thm:known-static-optimum}
Suppose $K\ge2$ and $\lambda>0$. The function $\psi_{\mathrm K}(\eta,\lambda)$ is strictly convex in $\eta\in[0,1)$. Its unique minimizer is
\begin{equation}
\eta_{\mathrm K}^\star(\lambda)
=
\begin{cases}
0,&\lambda\le K,\\[1ex]
\displaystyle
\frac{q_{\mathrm K}(\lambda)-1}{\lambda-1+q_{\mathrm K}(\lambda)},
&\lambda>K,
\end{cases}
\qquad
q_{\mathrm K}(\lambda):=\sqrt{\frac{\lambda-1}{K-1}}.
\label{eq:known-eta-star}
\end{equation}
Therefore coarse supervision enters the optimal static design if and only if $\lambda>K$.
\end{theorem}

\begin{proof}
Differentiating \Cref{eq:known-budget-risk} gives $\partial_\eta\psi_{\mathrm K}(0,\lambda)=K-\lambda$, while the second derivative is strictly positive on $[0,1)$. Hence the all-fine boundary is optimal exactly when $\lambda\le K$. Solving the first-order condition above the boundary gives \Cref{eq:known-eta-star}.
\end{proof}

The threshold has a direct economic interpretation: coarse labels must be sufficiently precise per unit cost along the aggregation direction to compensate for the $K-1$ response directions that still require fine labels. Define
$\Phi_{\mathrm K}^\star:=c_{\mathrm F}\sigma_{\mathrm F}^2d\min_{\eta\in[0,1)}\psi_{\mathrm K}(\eta,\lambda)$; its cumulative-risk interpretation is formalized together with the unknown-aggregation benchmark in \Cref{cor:static-benchmark-attainment}.

\subsection{Value of Coarse Supervision}
\label{subsec:known-value}

Let $\psi_{\mathrm K}^\star(\lambda):=\psi_{\mathrm K}(\eta_{\mathrm K}^\star(\lambda),\lambda)$. Below the break-even boundary, $\psi_{\mathrm K}^\star(\lambda)=K$. Above it,
\begin{equation}
\psi_{\mathrm K}^\star(\lambda)
=
\frac{\left(\sqrt{(K-1)(\lambda-1)}+1\right)^2}{\lambda},
\qquad \lambda>K.
\label{eq:known-optimal-coefficient}
\end{equation}
We measure the leading value of optimal coarse acquisition relative to all-fine by
\begin{equation}
G_{\mathrm K}(\lambda)
:=
1-\frac{\psi_{\mathrm K}^\star(\lambda)}{K}.
\label{eq:known-gain}
\end{equation}
Thus $G_{\mathrm K}=0.04$, for example, means a $4\%$ leading reduction in prediction risk at fixed budget, equivalently a $4\%$ leading reduction in annotation cost at fixed target risk. The gain is zero for $\lambda\le K$, strictly positive above the boundary, and satisfies
\begin{equation}
0\le G_{\mathrm K}(\lambda)<\frac1K,
\qquad
\lim_{\lambda\to\infty}G_{\mathrm K}(\lambda)=\frac1K.
\label{eq:known-gain-ceiling}
\end{equation}
Even arbitrarily cheap or precise coarse labels therefore cannot replace the fine information needed in the remaining $K-1$ response directions.

\section{Unknown Aggregation and Online Adaptation}
\label{sec:online}
We now study the main setting in which the aggregation vector $w_\star$ is unknown. The section first identifies the statistical price of this nuisance parameter and then shows that the resulting optimal resolution mix can be learned online without an additional first-order penalty.

\subsection{Unknown Aggregation: Static Benchmark}
\label{subsec:unknown-static}

When $w_\star$ is unknown, a coarse observation depends jointly on the fine map and aggregation direction. Profiling this nuisance changes the two information-block dimensions. Using the same $\eta$ and $\lambda$ as in \Cref{sec:static}, let $A_{\mathrm U}$ and $D_{\mathrm U}$ denote the fine-only and fine-plus-coarse dimensions, and define
\begin{equation}
A_{\mathrm U}:=(K-1)(d+1),
\qquad
D_{\mathrm U}:=d-K+1,
\qquad
\psi_{\mathrm U}(\eta,\lambda)
:=
\frac{A_{\mathrm U}/d}{1-\eta}
+
\frac{D_{\mathrm U}/d}{1+(\lambda-1)\eta}.
\label{eq:unknown-objective}
\end{equation}
Let $\eta_{\mathrm U}^\star(\lambda)$ minimize $\psi_{\mathrm U}(\eta,\lambda)$ over $\eta\in[0,1)$.

\begin{theorem}[Unknown-aggregation frontier]
\label{thm:unknown-frontier}
Under \Cref{ass:covariates,ass:noise-cost,ass:identifiability}, the first-order local minimax fine-prediction risk of an exogenous static allocation with budget $B$ and coarse spending share $\eta$ satisfies  
\begin{equation}
\mathcal R_{\mathrm U}(B,\eta)
=
\frac{c_{\mathrm F}\sigma_{\mathrm F}^2d}{B}
\psi_{\mathrm U}(\eta,\lambda)
+o(B^{-1}).
\label{eq:unknown-leading-risk}
\end{equation}
A projected cross-fitted one-step estimator attains the same leading term. Coarse supervision enters the optimal design if and only if
\begin{equation}
\lambda>\lambda_{\mathrm U}^{\mathrm{crit}}
:=
\frac{dK}{d-K+1}.
\label{eq:unknown-threshold}
\end{equation}
Compared with the known-$w_\star$ threshold $K$, unknown aggregation therefore converts exactly $K-1$ formerly coarse-informative directions into fine-only directions.
\end{theorem}

Define $\Phi_{\mathrm U}^\star:=c_{\mathrm F}\sigma_{\mathrm F}^2d\,\psi_{\mathrm U}(\eta_{\mathrm U}^\star,\lambda)$. Above the threshold,
$\eta_{\mathrm U}^\star=(q_{\mathrm U}-1)/(\lambda-1+q_{\mathrm U})$ with
$q_{\mathrm U}:=\sqrt{D_{\mathrm U}(\lambda-1)/A_{\mathrm U}}$.
The leading gain relative to all-fine is
$G_{\mathrm U}(\lambda):=1-\psi_{\mathrm U}(\eta_{\mathrm U}^\star,\lambda)/K$.
It is zero at and below the boundary, increases above it, and approaches the structural ceiling $(d-K+1)/(dK)$, compared with $1/K$ when $w_\star$ is known.

\textbf{Proof roadmap.}
The key step is \Cref{lem:unknown-information-split}: profiling out the $(K-1)$-dimensional simplex tangent of $w_\star$ moves exactly $K-1$ directions from the coarse-informative block into the fine-only block, giving the multiplicities in \Cref{eq:unknown-objective}. Taking the inverse-information trace yields \Cref{eq:unknown-leading-risk}. Strict convexity then reduces the design problem to one dimension: the derivative at $\eta=0$ gives the break-even boundary, while the first-order condition above the boundary gives the closed-form optimizer. Appendix~\ref{app:proof-unknown-frontier} gives the full information calculation, lower bound, and matching attainment.

\begin{lemma}[Nuisance information split]
\label{lem:unknown-information-split}
After profiling out local perturbations of $w_\star$ in the simplex tangent space, the efficient information for $\Theta_\star$ has two eigenspaces: $A_{\mathrm U}$ directions receive only fine-label information, while $D_{\mathrm U}$ directions receive both fine and coarse information.
\end{lemma}

\begin{proof}[Proof sketch]
A local perturbation $(H,a)$ of $(\Theta_\star,w_\star)$ changes the coarse mean through $Hw_\star+\Theta_\star a$. The tangent space of an interior simplex point has dimension $K-1$, and \Cref{ass:identifiability} makes its image under $\Theta_\star$ $K-1$ dimensional. Profiling over $a$ removes exactly those $K-1$ directions from the coarse-informative block. The full Schur-complement calculation appears in Appendix~\ref{app:proof-information-split}.
\end{proof}

Since first-order pointwise risk scales as $1/b$, these coefficients become $\log T$ cumulative-risk benchmarks on the budget clock.

\begin{corollary}[Static benchmark attainment]
\label{cor:static-benchmark-attainment}
Regular exogenous static procedures achieve $\mathfrak R_T^{\mathsf B_1}\le\Phi_{\mathrm K}^\star\log T+O(1)$ and $\mathfrak R_T^{\mathsf B_2}\le\Phi_{\mathrm U}^\star\log T+O(1)$; the two leading coefficients are local asymptotic minimax for their respective static experiments.
\end{corollary}

\subsection{Learning the Optimal Mix Online}
\label{subsec:online-adaptation}

The static optimum depends on the unknown $w_\star$ through $\lambda$. \Cref{alg:eet} therefore separates allocation learning from prediction: at the start of epoch $m$, the accumulated vanishing design stream is used to estimate $w_\star$ and the target coarse share, while the estimation stream tracks that share and fits the predictor without feeding its labels back into future allocation decisions. Within the estimation stream, \Cref{alg:eet} uses two deterministic folds; each fold takes one Fisher-scoring step from an opposite-fold pilot, and the two fine-map estimates are averaged and projected. Cross-fitting preserves pilot--score independence conditional on the design history while using estimation data at the full $b^{-1/2}$ scale.

\begin{algorithm}[!ht]  
\caption{Epochal Estimate-and-Track with Cross-Fitted Estimation (EET)}
\label{alg:eet}
\begin{algorithmic}[1]
\Require $B_0$, $\zeta_0\in(0,1/4)$, fixed design split $\eta_{\mathrm{des}}\in(0,1)$, costs/noise variances, and known regularity bounds $\tau,\kappa_\Theta,M_\Theta$
\State Run a fixed predictable initialization through budget $B_0$ containing both resolutions in the design stream and both estimation folds; initialize signed tracking imbalance $s\gets0$
\For{$m=0,1,2,\ldots$}
    \State Set $B_m=2^mB_0$, $\Delta B_m=B_{m+1}-B_m$, and $\zeta_m=\zeta_0/(m+2)^2$
    \State From accumulated design data available at $B_m$, fit guarded projected regressions $\widetilde\Theta_m,\widetilde u_m$ and set
    \[
    \widetilde w_m\in\argmin_{w\in\Delta_K,\;w_k\ge\tau}
    \|\widetilde u_m-\widetilde\Theta_mw\|_2^2,
    \qquad
    \widetilde\lambda_m=
    \frac{c_{\mathrm F}\sigma_{\mathrm F}^2}{c_{\mathrm C}\sigma_{\mathrm C}^2}
    \|\widetilde w_m\|_2^2.
    \]
    \State Set target coarse-budget share $\widehat\eta_m\gets\eta_{\mathrm U}^\star(\widetilde\lambda_m)$
    \State Reserve $\zeta_m\Delta B_m$ for the design stream and query it using the fixed split $\eta_{\mathrm des}$
    \While{the remaining estimation budget permits a query}
        \State $s_{\mathrm C}\gets s+(1-\widehat\eta_m)c_{\mathrm C}$,\quad
        $s_{\mathrm F}\gets s-\widehat\eta_m c_{\mathrm F}$
        \State Choose the affordable action whose updated imbalance has smaller absolute value (fine on ties); stop the block if the only affordable action increases $|s|$
        \State Assign the query to the next deterministic parity fold of that resolution, then observe $(x_t,Y_t^{A_t})$ and update $s$
        \State Recompute the projected two-fold cross-fitted one-step estimator from all estimation data collected so far
    \EndWhile
    \State Hold the current predictor fixed through the next design block until another estimation query completes
\EndFor
\end{algorithmic}
\end{algorithm}

The bounded-imbalance rule tracks $\widehat\eta_m$ to $O(1)$ cost. More importantly, conditional on the accumulated design history, the entire estimation resolution schedule and fold assignment are fixed before the associated covariates and labels are observed. Estimation labels therefore never feed back into future allocation decisions.
Let $w_0:=K^{-1}\mathbf 1$ and $\vartheta=(\operatorname{vec}(\Theta),v)$ be a local coordinate with $w=w_0+Qv$, and for the local statements let $\vartheta_\star$ be an interior point of the regular class with a sufficiently small radius-$\delta$ ball $\mathcal U_\delta$ around it.  

\begin{theorem}[First-order optimal online adaptation]
\label{thm:online-main}
Under \Cref{ass:covariates,ass:noise-cost,ass:identifiability,ass:sequential}, EET satisfies
\begin{equation}
\mathfrak R_T^{\mathrm{alg}}
\le
\Phi_{\mathrm U}^\star\log T+O(1).
\label{eq:online-upper}
\end{equation}
The bound is locally uniform: for sufficiently small $\delta$, $\sup_{\vartheta\in\mathcal U_\delta}\mathfrak R_T^{\mathrm{alg}}(\vartheta)\le(\Phi_{\mathrm U}^\star(\vartheta_\star)+o_\delta(1))\log T+O_\delta(1)$. Moreover,
\begin{equation}
\lim_{\delta\downarrow0}
\liminf_{T\to\infty}
\frac{1}{\log T}
\inf_{\pi,\widehat\Theta}
\sup_{\vartheta\in\mathcal U_\delta}
\mathfrak R_T(\pi,\widehat\Theta;\vartheta)
\ge
\Phi_{\mathrm U}^\star(\vartheta_\star),
\label{eq:online-lower}
\end{equation}
where the infimum ranges over predictable resolution policies and estimators. Hence EET is first-order locally asymptotically minimax.
\end{theorem}

Relative to the known-$w_\star$ oracle, the leading coefficient decomposes as $\Phi_{\mathrm U}^\star=\Phi_{\mathrm K}^\star+(\Phi_{\mathrm U}^\star-\Phi_{\mathrm K}^\star)$: the nonnegative difference is the nuisance price of unknown aggregation, while learning the allocation introduces no additional first-order term.

\textbf{Proof roadmap.}
For the upper bound, \Cref{lem:pilot-tracking} shows that the design stream uses only $O(b/(\log b)^2)$ budget and that the realized estimation-stream objective approaches the unknown-$w_\star$ static optimum fast enough that both the design opportunity cost and the induced allocation error are summable after the leading $1/b$ risk scaling. Conditional on this design history, \Cref{lem:conditional-efficient-estimation} fixes the resolution schedule and cross-fitting folds, so the predictor attains the efficient risk of the realized fine/coarse mix up to an $O(b_{\mathrm E}^{-3/2})$ remainder. Combining the two lemmas gives $L_b\le\Phi_{\mathrm U}^\star/b$ plus summable terms at estimation endpoints. Together with the held-predictor accounting in Appendix~\ref{app:proof-online-main}, summing over the full budget clock proves \Cref{eq:online-upper}; the same uniform bounds on a small regular neighborhood give the locally uniform statement.

For the lower bound, \Cref{lem:adaptive-information-lower} applies van Trees to a local prior on the joint parameter $(\Theta,w)$. Predictability makes the policy kernel parameter-free conditional on the past, so the same nuisance-profiled information frontier lower-bounds the per-budget Bayes risk by $\Phi_{\mathrm U}^\star/b$ up to lower-order terms. Summing yields \Cref{eq:online-lower}; hence the upper and lower bounds meet in their leading $\log T$ coefficient.

\begin{lemma}[Pilot and tracking error]
\label{lem:pilot-tracking}
Let $\Xi(b)$ be design-stream spending by budget $b$, let $b_{\mathrm E}:=b-\Xi(b)$, and let $\eta_b^{\mathrm E}$ be the realized coarse-budget share in the estimation stream. Uniformly over the regular parameter class,
\begin{equation}
\Xi(b)=O\!\left(\frac{b}{(\log b)^2}\right),
\qquad
\E\!\left[
\psi_{\mathrm U}(\eta_b^{\mathrm E},\lambda)
-
\psi_{\mathrm U}(\eta_{\mathrm U}^\star,\lambda)
\right]
=
O\!\left(\frac{(\log b)^3}{b}\right).
\label{eq:pilot-tracking-rate}
\end{equation}
No separation from the break-even boundary is required.
\end{lemma}

\begin{proof}[Proof sketch]
The fixed design split gives $\Theta(b/(\log b)^2)$ design observations of each resolution. Guarded projected regression therefore estimates $\lambda$ with mean-square error $O((\log b)^2/b)$. Uniform strong convexity of $\psi_{\mathrm U}$ makes the plug-in objective loss quadratic, including at the all-fine boundary. Geometric weighting across dyadic epochs and bounded-imbalance tracking give \Cref{eq:pilot-tracking-rate}. Appendix~\ref{app:proof-pilot-tracking} gives the full argument.
\end{proof}

\begin{lemma}[Cross-fitted efficient estimation]
\label{lem:conditional-efficient-estimation}
Let $\mathcal G_m$ be the design history determining the current epoch schedule. For sufficiently large estimation endpoints $b$,
\begin{equation}
L_b(\mathcal G_m)
:=
\E\!\left[\ell(\widehat\Theta_b,\Theta_\star)\mid\mathcal G_m\right]
\le
\frac{c_{\mathrm F}\sigma_{\mathrm F}^2d}{b_{\mathrm E}}
\psi_{\mathrm U}(\eta_b^{\mathrm E},\lambda)
+
O(b_{\mathrm E}^{-3/2}),
\label{eq:conditional-efficient-risk}
\end{equation}
with a constant uniform over the regular parameter class and the realized predictable schedules generated by EET.
\end{lemma}

At estimation endpoints, combining \Cref{lem:pilot-tracking,lem:conditional-efficient-estimation} and $b_{\mathrm E}^{-1}-b^{-1}=O(1/[b(\log b)^2])$ gives
\begin{equation}
L_b
\le
\frac{\Phi_{\mathrm U}^\star}{b}
+
O\!\left(
\frac{1}{b(\log b)^2}
+
\frac{(\log b)^3}{b^2}
+
\frac{1}{b^{3/2}}
\right),
\label{eq:per-budget-online-excess}
\end{equation}
whose remainder is summable over estimation endpoints. Appendix~\ref{app:proof-online-main} shows that the intervening held-predictor budget levels preserve summability on the full budget clock.

\begin{lemma}[Adaptive information lower bound]
\label{lem:adaptive-information-lower}
For any predictable resolution policy, a smooth compactly supported prior on the joint local parameter $(\Theta,w)$ inside $\mathcal U_\delta$ gives per-budget Bayes risk at least
\begin{equation}
\frac{\Phi_{\mathrm U}^\star(\vartheta_\star)-o_\delta(1)}{b}
-
O_\delta(b^{-2}).
\label{eq:adaptive-information-lower}
\end{equation}
\end{lemma}

\begin{proof}[Proof sketch]
Predictability makes the policy kernel parameter free, so adaptive Fisher information is the expected sum of queried-label information. Joint van Trees and nuisance profiling as in \Cref{lem:unknown-information-split} yield \Cref{eq:adaptive-information-lower}; Appendix~\ref{app:proof-adaptive-lower} gives the details.
\end{proof}

\section{Numerical Experiments}
\label{sec:experiments}
We use two synthetic instances to test complementary consequences of the theory: $d=6,K=5$ widens the known/unknown break-even gap, while $d=20,K=5$ is the canonical online instance. All losses are exact population fine-prediction risks under the whitened covariance; full protocols, estimators, seeds, and diagnostics appear in Appendix~\ref{app:experiments}.

\textbf{Structural effect of unknown aggregation.}
We first isolate the price of not knowing the aggregation direction. For $d=6$ and $K=5$, the known-aggregation break-even point is $\lambda=5$, whereas the unknown-aggregation boundary is $\lambda=15$. At budget $B=76{,}800$, the known-$w_\star$ design achieves fixed-budget gains relative to all-fine of $2.87\%$ (95\% CI $[1.35,4.40]\%$) at $\lambda=10$ and $4.66\%$ ($[3.03,6.30]\%$) at $\lambda=14$, compared with the closed-form targets $2.00\%$ and $3.68\%$; both targets lie inside the corresponding intervals. The unknown-$w_\star$ optimum remains all-fine throughout this intermediate region. Appendix~\ref{app:structural-static-experiment} reports the full sweep and both tested budgets.

\begin{figure}[t]
\centering
\includegraphics[width=0.75\textwidth]{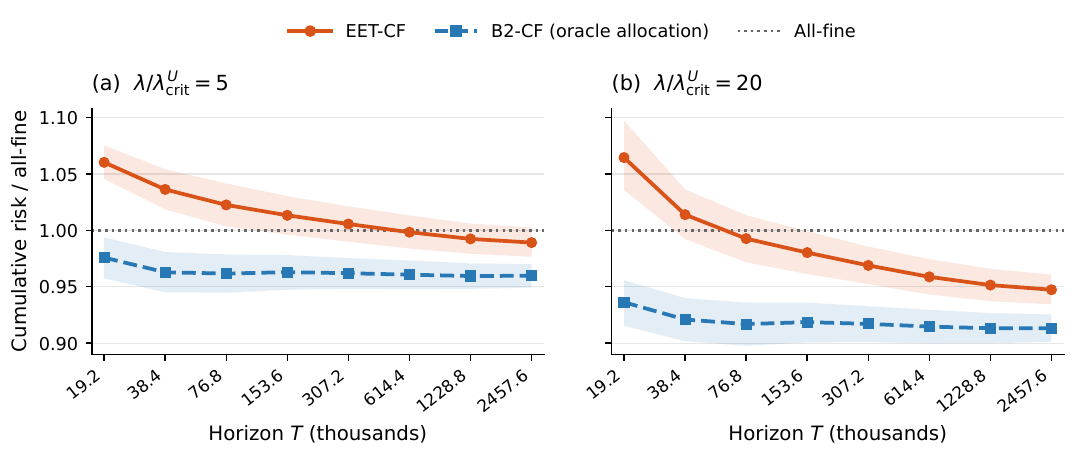}
\caption{\textbf{Finite-budget value of online adaptation.}
Cumulative fine-prediction risk is normalized by all-fine, so values below $1$ indicate an improvement. EET-CF is EET with the cross-fitted estimator in \Cref{alg:eet}; B2-CF uses the same estimator with the oracle unknown-aggregation static share and still treats $w_\star$ as unknown. Shaded bands are paired-bootstrap 95\% intervals over 20 seeds. In the $20\times$ regime EET-CF clearly improves on all-fine; in the $5\times$ regime the mean improvement appears later and remains smaller.}
\label{fig:main-online-value}
\end{figure}

\textbf{Online adaptation.}
For the online experiment we use $d=20$, $K=5$, $w_\star=(0.30,\allowbreak 0.25,\allowbreak 0.20,\allowbreak 0.15,\allowbreak 0.10)$, $c_{\mathrm F}=5$, $c_{\mathrm C}=1$, $\sigma_{\mathrm F}=1$, independent Rademacher covariates, and independent Gaussian query noise. We vary $\sigma_{\mathrm C}$ and report $\lambda/\lambda_{\mathrm U}^{\mathrm{crit}}\in\{5,20\}$ over 20 seeds. At $T=2{,}457{,}600$, $TL_T/\Phi_{\mathrm U}^\star$ is $1.031$ versus $1.028$ at $5\times$ and $1.026$ versus $1.023$ at $20\times$ for EET-CF versus B2-CF, showing that the learned allocation is numerically close to the oracle-share benchmark. Figure~\ref{fig:main-online-value} reports cumulative fine-prediction risk relative to all-fine across the tested horizons. EET-CF improves the final mean by $1.08\%$ at $5\times$, although the 95\% interval includes zero, and by $5.26\%$ at $20\times$ (95\% CI $[3.92,6.56]\%$). The slower cumulative approach is consistent with a bounded additive lower-order term, whose relative effect decays only logarithmically; Appendix~\ref{app:online-remainder-diagnostic} reports the direct remainder diagnostic. Below and at the break-even boundary, all-fine is the static optimum and the online learner retains a visible transient exploration cost; Appendix~\ref{app:below-threshold-diagnostic} reports that diagnostic.

\section{Discussion and Conclusion}
\label{sec:conclusion}

In this work, we study cost-aware two-resolution supervision when expensive fine labels reveal a vector response and cheaper coarse labels reveal an unknown scalar aggregation. The theory gives a closed-form break-even rule, optimal resolution mixes, and an online estimate-and-track policy that attains the optimal first-order cumulative-risk coefficient with a matching local asymptotic minimax lower bound. Experiments illustrate the structural cost of unknown aggregation and show EET approaching the oracle-share benchmark, with finite-budget gains when coarse supervision is favorable. The analysis isolates a linear-Gaussian setting with a fixed aggregation rule and pre-covariate resolution choice. We discuss richer extensions in Appendix~\ref{app:limitations}.

\clearpage

\bibliographystyle{plainnat}
\bibliography{references}

\clearpage
\appendix
\section*{Appendix Roadmap}
Appendix~\ref{app:discussion} expands the literature comparison and discusses limitations and extensions. Appendix~\ref{app:experiments} gives the complete synthetic protocols, uncertainty calculations, structural comparisons, and finite-budget diagnostics. Appendix~\ref{app:proofs} contains the proofs: Proposition~\ref{prop:known-static-risk} in Appendix~\ref{app:proof-known-risk}, Theorem~\ref{thm:known-static-optimum} in Appendix~\ref{app:proof-known-optimum}, Theorem~\ref{thm:unknown-frontier} in Appendix~\ref{app:proof-unknown-frontier}, Lemma~\ref{lem:unknown-information-split} in Appendix~\ref{app:proof-information-split}, Corollary~\ref{cor:static-benchmark-attainment} in Appendix~\ref{app:proof-static-benchmark}, and the online theorem and its three supporting lemmas in Appendices~\ref{app:proof-online-main}--\ref{app:proof-adaptive-lower}.

\paragraph{Frequently used notation.}
\begin{center}
\small
\begin{tabular}{ll}
\toprule
Symbol & Meaning \\
\midrule
$B,b,T$ & static budget, running budget level, and cumulative-risk horizon \\
$B_0$ & end of finite initialization and start of the reported budget clock \\
$\mathsf B_1,\mathsf B_2,\mathsf B_3$ & known-$w_\star$ oracle, unknown-aggregation oracle-share, online learned allocation \\
$N_{\mathrm F},N_{\mathrm C}$ & fine and coarse query counts \\
$\eta,\widehat\eta_m,\eta_b^{\mathrm E}$ & static, epoch-target, and realized estimation-stream coarse shares \\
$\lambda$ & effective coarse-information ratio in \Cref{eq:known-lambda} \\
$\psi_{\mathrm K},\psi_{\mathrm U}$ & known/unknown static risk objectives \\
$G_{\mathrm K},G_{\mathrm U}$ & known/unknown leading gains relative to all-fine \\
$A_{\mathrm U},D_{\mathrm U}$ & fine-only and fine-plus-coarse multiplicities under unknown aggregation \\
$\Phi_{\mathrm K}^\star,\Phi_{\mathrm U}^\star$ & optimal known/unknown first-order cumulative-risk coefficients \\
$Q,w_0,v,\vartheta$ & simplex tangent basis, center, coordinate, and joint local parameter \\
$\mathcal U_\delta$ & local regular neighborhood around $\vartheta_\star$ \\
$B_m,\zeta_m$ & dyadic epoch boundary and design-stream budget fraction \\
$\Xi(b),b_{\mathrm E}$ & design-stream spending and remaining estimation budget \\
$\mathcal G_m$ & design-stream history determining epoch-$m$ allocation \\
$L_x$ & almost-sure whitened covariate-norm bound \\
\bottomrule
\end{tabular}
\end{center}

\section{Extended Discussion}
\label{app:discussion}

This section expands the literature comparison and the limitations discussed briefly in the main text, focusing on aspects that interact directly with the two-resolution information structure.

\subsection{Extended Related Work}
\label{app:extended-related}

\paragraph{Multiresponse optimal design and heterogeneous measurements.}
Classical and modern optimal-design theory provides the broad statistical framework for combining heterogeneous measurements. \citet{silvey1973geometric} develops a geometric approach to optimal design, while \citet{sagnol2011computing} shows that several multiresponse criteria, including A-optimality, can be handled under linear resource constraints. Related work considers multivariate observations and increasingly general response structures \citep{soumaya2015optimal,somogyi2026randomized,filova2026optimal,sagnol2015computing}. In the known-aggregation benchmark of our model, vectorizing $\Theta$ turns a fine query and a coarse query into two known information matrices, so the static problem is a structured two-experiment instance of this general framework. The special rank-one response structure reduces the cost-constrained design to a scalar budget share and yields the explicit threshold in \Cref{sec:static}. Our main unknown-aggregation setting differs because the coarse observation operator itself contains an unknown nuisance parameter.

A related measurement-design viewpoint appears in \citet{mareis2026cost}, who optimize complete and partial measurements under collection costs for a causal estimand. This work reinforces the point that a full-versus-partial acquisition decision can be statistically central without being specific to active learning. Our target and observation geometry are different: the target here is the full multivariate regression map, while the cheaper observation is a scalar projection of that same map.

\paragraph{Multifidelity estimation and adaptive resource allocation.}
Multifidelity methods also allocate finite budgets across expensive and cheap information sources. \citet{peherstorfer2016optimal} derives analytic cost-optimal model-management rules for multifidelity Monte Carlo, while \citet{schaden2020multilevel} and \citet{croci2023multi} study budget-optimal linear unbiased estimation and multi-output extensions. The relationship between fidelities need not be known in advance: \citet{peherstorfer2019multifidelity} adapts a low-fidelity model, \citet{prescott2024efficient} estimates allocation-relevant quantities during multifidelity inference, \citet{xu2022bandit} learns cross-model statistical relationships while allocating a budget, and \citet{dixon2026optimally} balances pilot estimation against the final multifidelity estimator.

These papers are close to our online problem at the level of resource allocation: all must learn quantities that determine whether a cheap source is useful. The inferential distinction is that $w_\star$ is part of the conditional-mean observation law $u_\star=\Theta_\star w_\star$, so uncertainty about the cross-resolution relationship changes the efficient information for the target itself. Profiling out $w_\star$ removes exactly $K-1$ coarse-informative directions and changes the static break-even boundary. Online optimal design gives another neighboring perspective: \citet{fontaine2021online} learns an {A}-optimal scalar-regression design under unknown heteroscedastic noise, whereas our acquisition action chooses response resolution before observing the current covariate. More broadly, related online decision problems couple learning with costly resource or action decisions \citep{xu2025joint,xu2025online}, or with unknown response laws in contextual pricing \citep{xu2026optimal}.

\paragraph{Cost-sensitive active learning and annotation granularity.}
Cost-sensitive active learning has long treated the annotation source as part of the acquisition decision. \citet{donmez2008proactive} select among imperfect oracles while accounting for cost and reliability, and \citet{wallace2011should} allocate instances across experts of different quality and cost. Other work varies the information requested from the same example: \citet{hu2018active} studies partial-feedback queries, \citet{yan2018cost} considers hierarchy-dependent query costs, and \citet{rotman2022multi} allows task-specific annotations with different granularities and costs. More recent work makes supervision level especially explicit, including full versus weak annotations \citep{tejero2023full} and instance-wise supervision-level optimization \citep{matsuo2025instance}. High-stakes chatbot evaluation also combines component-level and end-to-end assessment under limited expert supervision \citep{jha2026developing}, providing an application-level motivation for heterogeneous supervision.

A related literature studies inference from intrinsically coarse labels. \citet{fotakis2021efficient} formalizes categorical coarsening through partitions of a latent fine-label space and derives learning guarantees that depend on the information preserved by coarsening. \citet{nguyen2026leveraging} considers expensive fine-grained human labels together with cheaper coarse VLM labels in active learning. These settings motivate resolution-aware acquisition but use different observation structures. Here the learner purchases either a $K$-dimensional response or a scalar projection of that same response, while the target remains the full vector-valued conditional mean.

\subsection{Limitations and Extensions}
\label{app:limitations}

\paragraph{Joint item and resolution selection.}
The canonical model chooses the resolution before observing the current covariate. This isolates the global value of each resolution while retaining covariates in the prediction problem itself. If $x_t$ were available before the action, the learner would need to decide both \emph{where} information is most valuable and \emph{which resolution} to purchase. The resulting problem is a matrix-valued optimal-design problem rather than the one-dimensional frontier derived here.

\textbf{Richer cross-resolution structure.}
Several extensions change the cross-resolution information geometry, including shared item-level effects, multiple or heterogeneous coarse channels, nonlinear aggregation, and unknown noise levels; we discuss these separately below. Near the asymptotic break-even boundary, finite-budget benefits can also be delayed by estimation and adaptation costs.

\paragraph{Shared item-level noise and repeated measurements.}
The action space in the present paper requests only one resolution for each arriving item, so no joint fine/coarse noise law is needed for the main likelihood. If an extension allowed both resolutions to be collected on the same item, a shared latent item effect or correlated measurement noise could make pairing informative beyond two separately sampled observations. Such a model would also make the effective coarse variance depend on the form of the shared component and could require learning additional noise parameters.

\paragraph{Multiple resolutions and heterogeneous aggregation.}
We analyze one fine channel and one scalar coarse channel with a single fixed aggregation vector. Multiple coarse channels, several levels of granularity, or annotator-dependent aggregation vectors would replace the single rank-one direction by a collection of partially overlapping information subspaces. The same efficient-information viewpoint may still apply, but the phase diagram could contain several entry thresholds rather than one scalar break-even point.

\paragraph{Scale normalization and unconstrained aggregation.}
The simplex assumption fixes the scale of the coarse aggregation and gives a $(K-1)$-dimensional nuisance tangent. If $w_\star$ were instead locally unconstrained in $\mathbb R^K$ and $\Theta_\star$ had full column rank with $K<d$, the same tangent-space calculation gives
\[
A=d(K-1)+K,
\qquad
D=d-K,
\qquad
\lambda_{\mathrm{crit}}=\frac{dK}{d-K}.
\]
Thus the efficient-information argument and scalar allocation optimization carry over, but with a larger nuisance price. We do not make a corresponding intercept claim because an additional offset interacts differently with the centered covariate model and requires a separate derivation.

\paragraph{Nonlinear aggregation.}
The exact decomposition relies on the coarse conditional mean being linear in the fine response. With a smooth nonlinear aggregator, the locally informative direction can depend on the current parameter and potentially on the item itself. The coarse-information subspace would then vary across observations, and the static design would generally no longer collapse to a single scalar allocation share.

\paragraph{High-dimensional structure.}
The theory keeps $d$ and $K$ fixed as the annotation budget grows. If either dimension grows, or if $\Theta_\star$ is sparse or low rank, projected least squares and classical efficiency must be replaced by regularized estimators and a corresponding high-dimensional analysis. Structural assumptions may also change the number of directions that coarse feedback can usefully inform.

\paragraph{Estimated covariance and noise levels.}
The main theorem treats $\Sigma_x$ and the two noise variances as known; exact whitening is justified by free unlabeled covariates. The allocation depends on noise and cost through their cost-adjusted ratio. If an independent or cross-fitted estimator of that ratio has mean-square error no larger than the aggregation-weight contribution, the same quadratic plug-in argument suggests that the first-order allocation result should persist. We leave a complete theorem for estimated variances to future work.

\paragraph{Expected rather than time-uniform guarantees.}
Our main performance criterion is ex-ante expected cumulative fine-prediction risk. A time-uniform or high-probability guarantee would require simultaneous control of the sequential score, the pilot, and the allocation error over all budget levels. Such a result is stronger and may introduce additional logarithmic factors even if the first-order expected-risk coefficient is unchanged.

\section{Extended Numerical Experiments}
\label{app:experiments}

This section records the full synthetic protocol, uncertainty calculations, and additional diagnostics behind \Cref{sec:experiments}. The experiments test quantitative consequences of the linear theory; they are not evidence that the canonical model is an accurate description of a particular real annotation system.

\subsection{Experimental Setup and Reproducibility}
\label{app:experimental-setup}

\paragraph{Canonical $d=20,K=5$ instance.}
The static-convergence and online experiments use
$w_\star=(0.30,0.25,0.20,0.15,0.10)$,
$c_{\mathrm F}=5$, $c_{\mathrm C}=1$, and $\sigma_{\mathrm F}=1$.
The columns of $\Theta_\star$ are the first five columns of a deterministic orthonormal DCT-II basis in $\mathbb R^{20}$, so $\Theta_\star^\top\Theta_\star=I_5$. Covariates have independent Rademacher coordinates, and queried observations use independent Gaussian noise. For this instance,
$A_{\mathrm U}=84$, $D_{\mathrm U}=16$, and $\lambda_{\mathrm U}^{\mathrm{crit}}=6.25$. We vary only $\sigma_{\mathrm C}$ to obtain the desired effective ratios.

The static experiments use 200 seeds, $42001$--$42200$, common random numbers across allocations, and exact population fine-prediction risk. The larger-budget confirmation uses $B\in\{19200,38400,76800\}$ together with the earlier budgets $2400,4800,9600$; the exact theory-optimal allocations are evaluated in the $2\times$, $5\times$, and $20\times$ regimes.

The online experiment uses exactly 20 seeds, $44001$--$44020$, initialization budget $B_0=9600$, dyadic epochs, design share
$\zeta_m=0.20/(m+2)^2$, and a fixed $50\%$ coarse budget share inside each design block. The reported regimes are
$\lambda/\lambda_{\mathrm U}^{\mathrm{crit}}\in\{0.75,1,5,20\}$. All four use horizons through $614400$; the two mixed regimes additionally use $1228800$ and $2457600$.

\paragraph{Initialization and cumulative-risk clock.}
For the canonical online instance, initialization through $B_0$ is fixed in advance: 40 design-fine labels, 40 design-coarse labels, 40 mandatory estimation-fine labels, and 40 mandatory estimation-coarse labels are collected first, after which the remaining pre-$B_0$ budget purchases 1,824 additional estimation-fine labels. Thus the design stream contains 40 fine and 40 coarse observations (cost 240), the estimation stream contains 1,864 fine and 40 coarse observations (cost 9,360), and there is no unused residue. The design observations are used only to form the first allocation target; the predictor at $B_0$ is fit from the estimation folds only. The reported cumulative risk starts inclusively at $B_0$:
$\mathfrak R_T=\sum_{b=B_0}^T L_b$.

\paragraph{Cross-fitted EET estimator.}
The design stream estimates $w_\star$ and the target allocation. Estimation observations are permanently assigned by deterministic parity within each resolution stream to two folds before the current label is observed. At every completed estimation query, each score fold takes one expected-Fisher scoring step from a guarded projected pilot fitted on the opposite fold; the two fine-map estimates are averaged and projected onto the common operator-norm ball. The predictor is held fixed between estimation-query completions, including throughout design blocks. The static $\mathsf B_2$ comparator uses the same cross-fitted estimator but is given only the locally optimal unknown-aggregation static share; its estimator still treats $w_\star$ as unknown. At the all-fine boundary the efficient estimator reduces to pooled guarded projected fine regression, so we use that estimator rather than artificially splitting the fine-only data.

\paragraph{Uncertainty.}
For online curves, each point is a ratio of aggregate mean cumulative risks. The shaded intervals are pointwise 95\% percentile intervals from 20,000 paired bootstrap resamples of the 20 common seeds, using fixed bootstrap seed 48001. Paired method differences use seedwise differences with mean $\pm1.96$ standard errors. For static gains, the primary quantity is the ratio-of-means gain $1-\bar L_{\mathrm{method}}/\bar L_{\mathrm{AF}}$; its 95\% interval uses a paired delta-method calculation. The cumulative risk is the exact integer-budget sum of held population risk, not a checkpoint interpolation or test-set estimate. No scoring guard, fallback, or final-projection event occurred in the online estimator evaluations.  

\subsection{Static Convergence Check}
\label{app:static-convergence}

The first static experiment asks whether the finite-budget unknown-aggregation risk approaches the closed-form coefficient from \Cref{thm:unknown-frontier}. All six budgets use the same two-fold cross-fitted projected one-step estimator as the static $\mathsf B_2$ construction. Figure~\ref{fig:app-static-convergence} reports paired uncertainty over the 200 seeds. At $B=76800$, the empirical gains in the $2\times$, $5\times$, and $20\times$ regimes are $1.226\%$, $4.848\%$, and $9.617\%$, compared with theoretical values $1.548\%$, $5.271\%$, and $10.012\%$. The last-step changes in coefficient excess are much smaller than their paired standard errors, so the slight visual flattening or increase at the largest budgets is not resolved statistically.

\begin{figure}[t]
\centering
\includegraphics[width=0.94\textwidth]{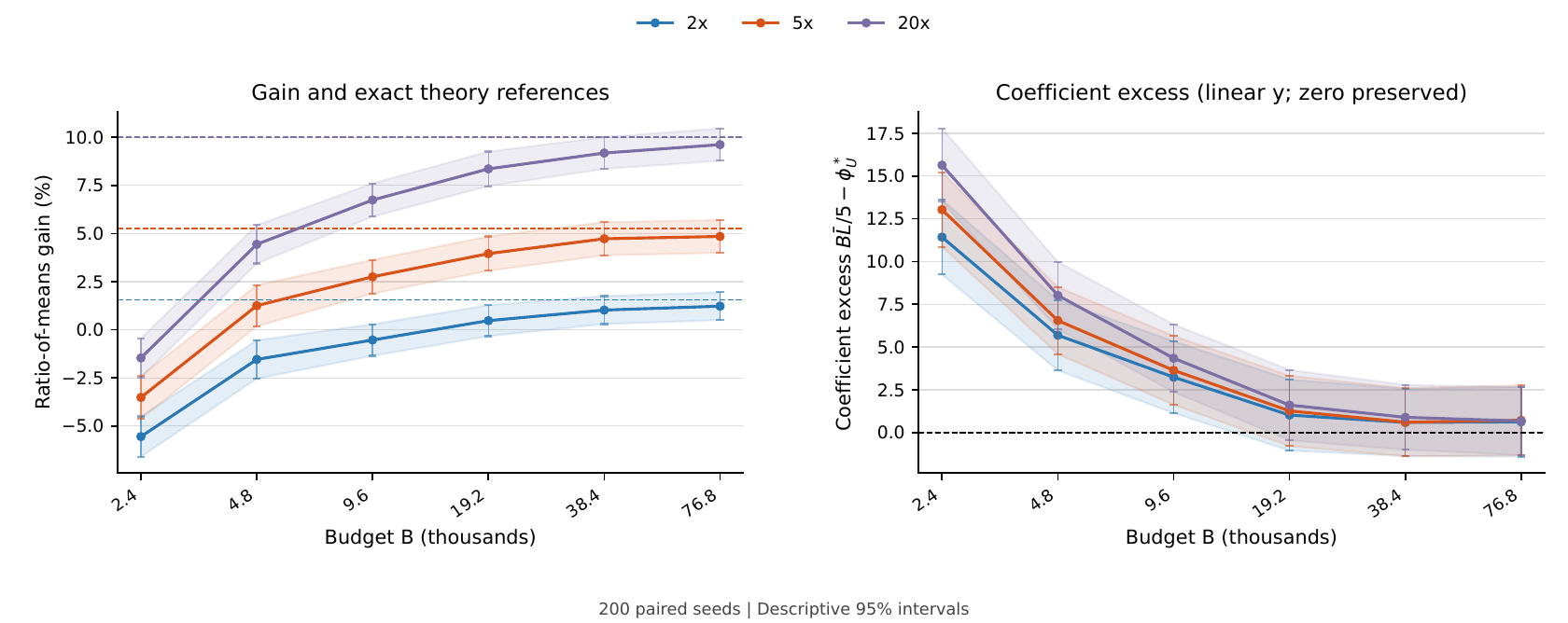}
\caption{\textbf{Static convergence in the canonical $d=20,K=5$ instance.} Left: ratio-of-means annotation gain with paired 95\% intervals and exact theory references. Right: first-order coefficient excess with descriptive 95\% intervals. All points use the same 200 paired seeds.}
\label{fig:app-static-convergence}
\end{figure}

\subsection{Known versus Unknown Aggregation}
\label{app:structural-static-experiment}

To make the structural cost of unknown aggregation visible, we use a second instance with $d=6$ and $K=5$. The first five DCT-II basis vectors define $\Theta_\star$; $w_\star$, costs, and fine-noise variance are unchanged. Here the known-aggregation break-even point is $5$, whereas the unknown-aggregation threshold is $15$. We evaluate
$\lambda\in\{4,5,7.5,10,14,15,20,40\}$ at budgets $19200$ and $76800$ using 200 seeds. $\mathsf B_1$ receives the known-$w_\star$ theory-optimal share and uses efficient known-weight Gaussian regression; $\mathsf B_2$ receives the unknown-$w_\star$ theory-optimal share and uses the cross-fitted estimator.

Figure~\ref{fig:app-known-unknown} shows both budgets. The vertical axis is the fixed-budget gain relative to all-fine,
$1-L_{\mathrm{method}}/L_{\mathrm{AF}}$: a value of $4\%$ means $4\%$ lower fine-prediction risk at the same annotation budget. In the leading theory, the same percentage is also the annotation-cost saving at a fixed target risk. The shaded interval $5<\lambda\le15$ is the regime in which known aggregation recommends a mixed design while unknown aggregation recommends all-fine acquisition. The exact known-$w_\star$ gains are $2.000\%$ at $\lambda=10$ and $3.683\%$ at $\lambda=14$. At $B=76800$, the corresponding empirical gains are $2.874\%$ (95\% CI $[1.352,4.397]\%$) and $4.664\%$ ($[3.032,6.296]\%$); both theory targets lie inside the empirical intervals. $\mathsf B_2$ is exactly all-fine throughout this intermediate region by the pre-specified theory-optimal allocation. The experiment illustrates the consequence of the analytic thresholds; it does not independently estimate them from an allocation grid.

\begin{figure}[t]
\centering
\includegraphics[width=0.92\textwidth]{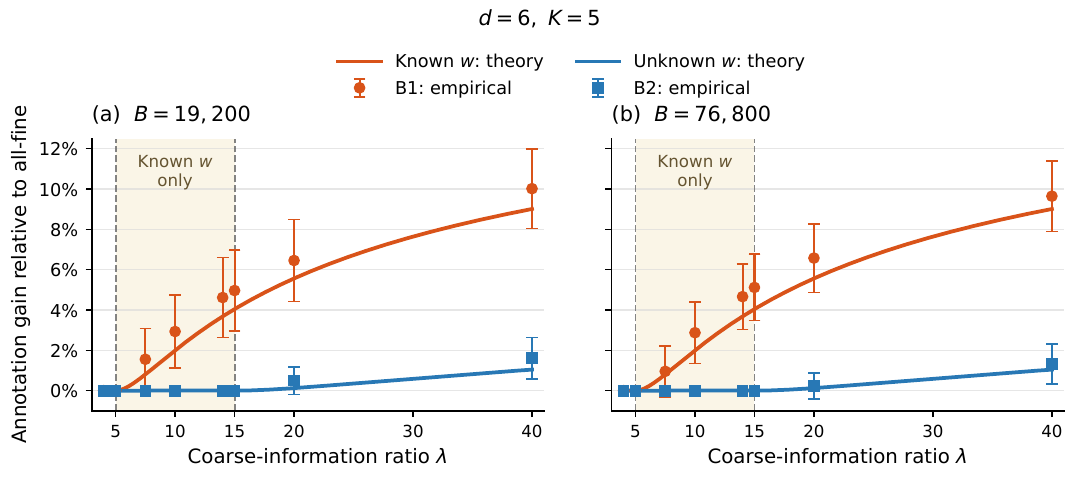}
\caption{\textbf{Known versus unknown aggregation, $d=6$, $K=5$.} Theory curves and empirical gains are shown at two budgets. The shaded region lies between the known-aggregation threshold $5$ and the unknown-aggregation threshold $15$. Error bars are 95\% intervals over 200 seeds.}
\label{fig:app-known-unknown}
\end{figure}

The size of this structural gap depends strongly on $d/K$. For fixed $K=5$, the unknown threshold $5d/(d-4)$ equals $25$ at $d=5$, $15$ at $d=6$, $10$ at $d=8$, $6.25$ at $d=20$, and approaches the known threshold $5$ as $d$ grows. Equivalently, the relative loss in the maximal gain ceiling is $(K-1)/d$.

\subsection{Online Allocation Learning and Finite-Budget Value}
\label{app:online-experiments}

The main text reports cumulative risk relative to all-fine. The allocation path itself depends only on the independent design stream, not on estimation labels. Figure~\ref{fig:app-allocation-learning} shows that the learned and realized coarse shares approach the distinct theoretical targets $0.042891$ and $0.029955$ in the $5\times$ and $20\times$ regimes. The smaller spending share at $20\times$ is not paradoxical: when each coarse dollar is much more informative, less coarse spending is needed to complement the fine-only directions.

\begin{figure}[t]
\centering
\includegraphics[width=0.96\textwidth]{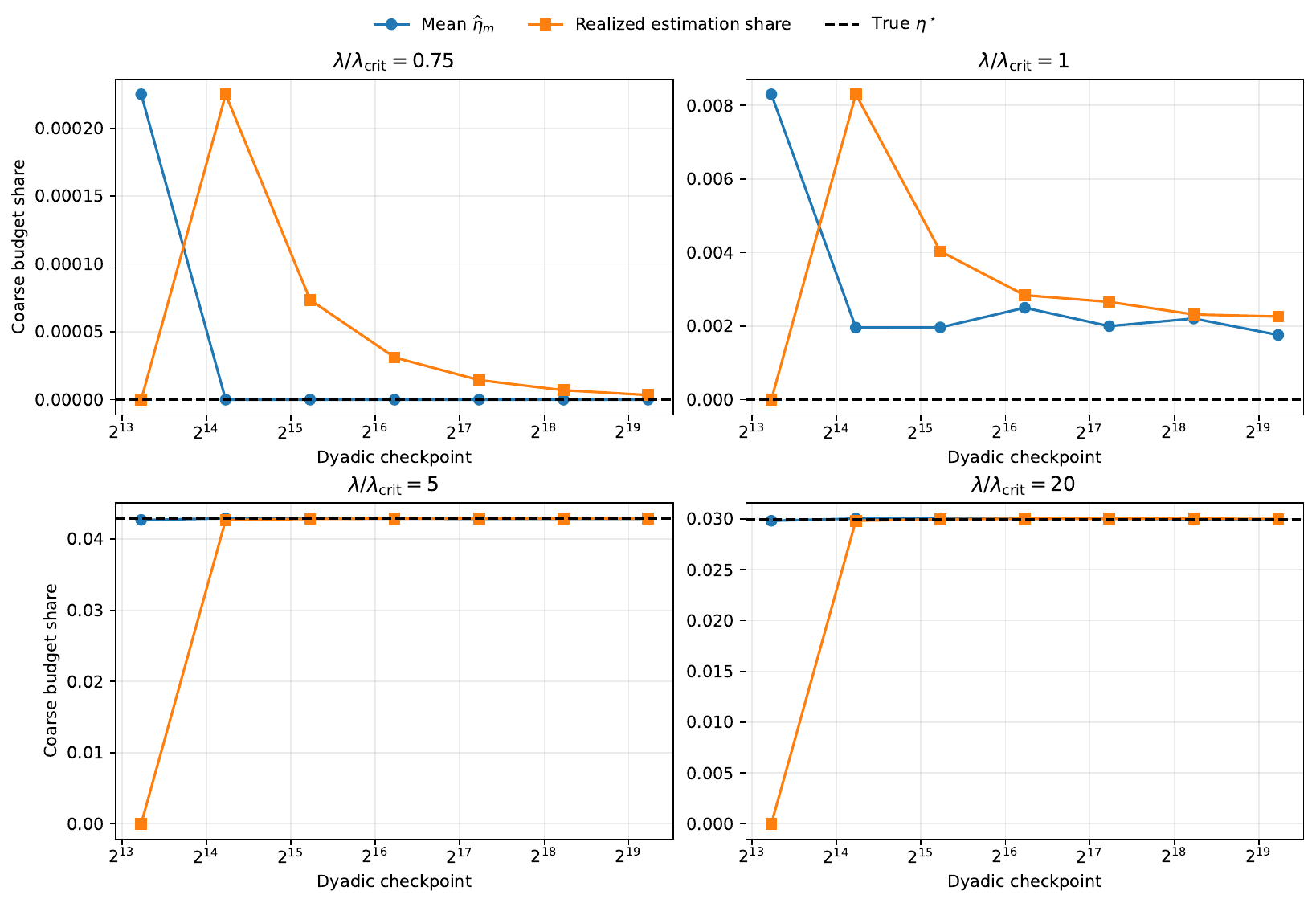}
\caption{\textbf{Allocation learning.} The design-stream plug-in target and the realized estimation-stream coarse-budget share track the theoretical optimal share. Estimation labels never feed back into the allocation rule.}
\label{fig:app-allocation-learning}
\end{figure}

At the largest tested horizon, EET-CF is numerically close to $\mathsf B_2$ in \emph{instantaneous} first-order coefficient even though a positive cumulative gap remains from earlier adaptation. Table~\ref{tab:app-online-summary} summarizes the two mixed regimes. The $\mathsf B_2$ column is an oracle-share reference that EET is intended to approach, not beat.

\begin{table}[t]
\centering
\small
\begin{tabular}{lccccc}
\toprule
Regime & Final EET coeff. & Final B2 coeff. & EET/all-fine & Mean crossover & 95\% crossover \\
\midrule
$5\times$ & $1.031$ & $1.028$ & $0.989$ & $614{,}400$ & none \\
$20\times$ & $1.026$ & $1.023$ & $0.947$ & $76{,}800$ & $153{,}600$ \\
\bottomrule
\end{tabular}
\caption{\textbf{Online summary at $T=2{,}457{,}600$.} The coefficient columns report $TL_T/\Phi_{\mathrm U}^\star$. ``Mean crossover'' is the first tested horizon at which mean cumulative EET risk is below all-fine; ``95\% crossover'' additionally requires the paired pointwise 95\% interval for EET minus all-fine to be entirely below zero.}
\label{tab:app-online-summary}
\end{table}

At the final horizon, the paired cumulative EET-minus-$\mathsf B_2$ gaps are $80.42$ (95\% CI $[59.94,100.90]$) and $93.26$ ($[79.38,107.15]$) in the $5\times$ and $20\times$ regimes. A nonzero finite cumulative gap is compatible with equal leading $\log T$ coefficients; the theorem does not require the difference to vanish.

\subsection{First-Order Remainder Diagnostic}
\label{app:online-remainder-diagnostic}

To inspect the lower-order term directly, we subtract the discrete first-order benchmark
$H_T:=\sum_{b=B_0}^T b^{-1}$. Figure~\ref{fig:app-bounded-remainder} plots
$\mathfrak R_T-\Phi^\star H_T$, using $\Phi_{\mathrm U}^\star$ for EET-CF and $\mathsf B_2$ and the all-fine coefficient $dKc_{\mathrm F}\sigma_{\mathrm F}^2$ for all-fine. Over the last three horizons, all consecutive-increment 95\% intervals contain zero. The displayed tails are therefore visually consistent with flattening, but with 20 paths they do not distinguish a bounded plateau from modest continuing drift. We use this figure only as a finite-horizon diagnostic, not as empirical proof of the $O(1)$ remainder.

\begin{figure}[t]
\centering
\includegraphics[width=0.94\textwidth]{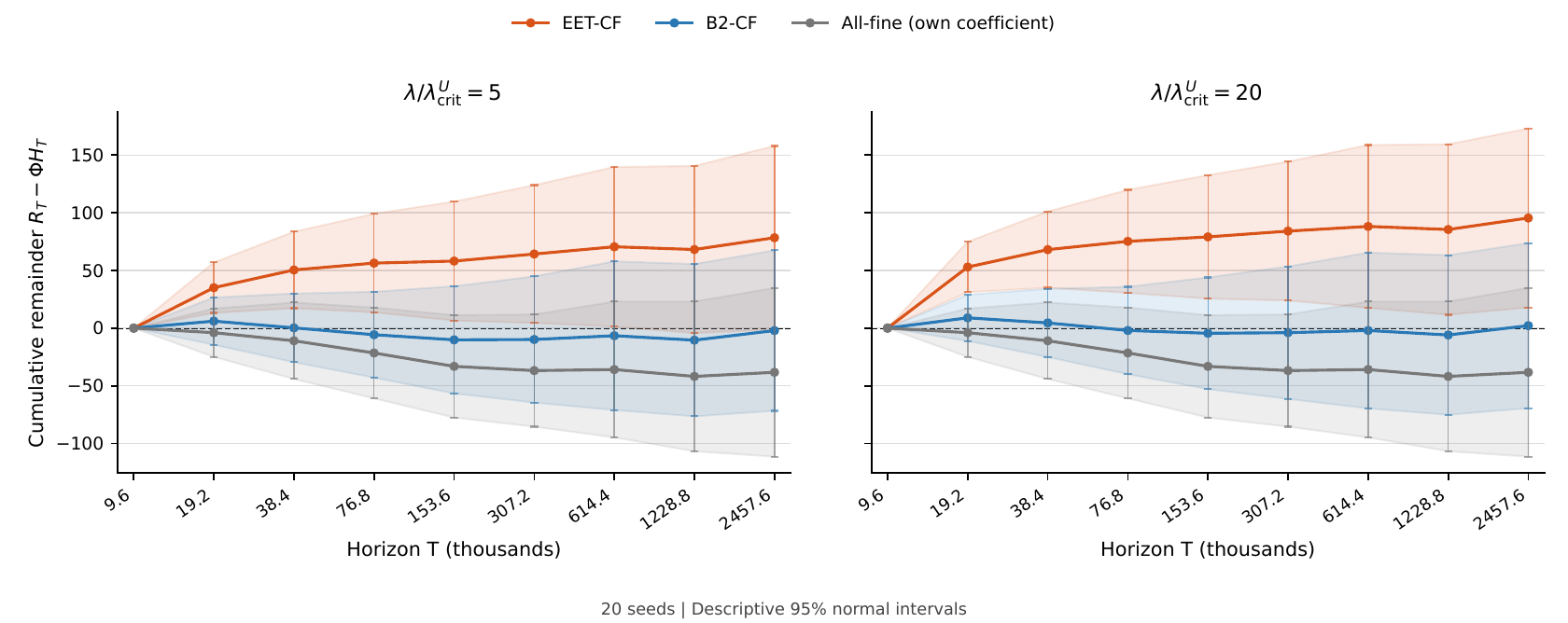}
\caption{\textbf{First-order remainder diagnostic.} Cumulative risk after subtracting the exact discrete harmonic first-order benchmark. Bands are descriptive pointwise 95\% intervals over the 20 saved paths. The late trajectories are consistent with a bounded lower-order term but remain too noisy to establish a plateau.}
\label{fig:app-bounded-remainder}
\end{figure}

\subsection{Below-Threshold Finite-Budget Diagnostic}
\label{app:below-threshold-diagnostic}

At $0.75\lambda_{\mathrm U}^{\mathrm{crit}}$ and $\lambda_{\mathrm U}^{\mathrm{crit}}$, all-fine is the static optimum. We therefore compare normal EET-CF with a diagnostic variant that forces the post-initialization estimation target to zero while preserving the same design-stream spending rule. Figure~\ref{fig:app-below-threshold} shows that forcing zero does not remove the roughly $4\%$ cumulative overhead at $T=614400$. In the $0.75\times$ regime, 18 of 20 normal and forced-zero post-initialization acquisition paths are identical; at the exact boundary they differ more often because finite-sample plug-in estimates cross the threshold.

\begin{figure}[t]
\centering
\includegraphics[width=0.92\textwidth]{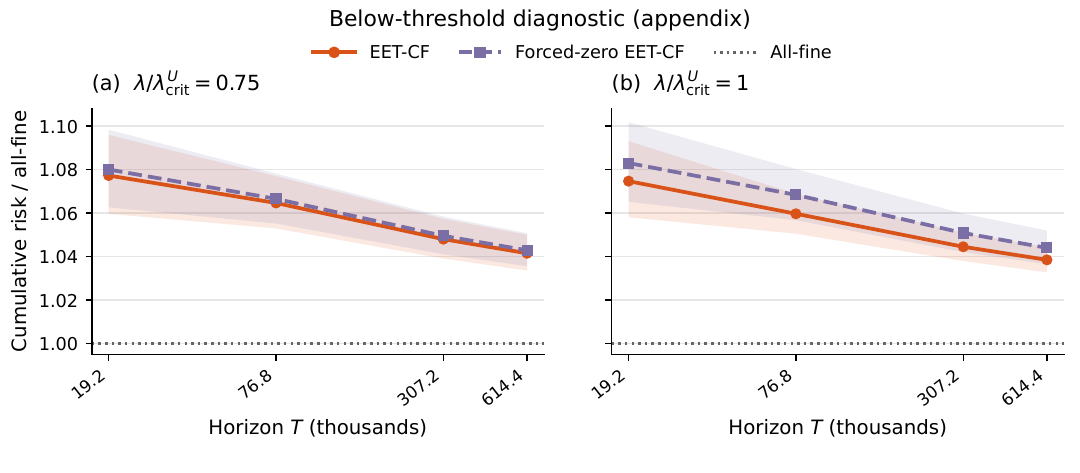}
\caption{\textbf{Below-threshold finite-budget diagnostic.} Cumulative risk is normalized by all-fine. Forced-zero acquisition preserves design spending but removes positive post-initialization target shares. Its failure to reduce the gap indicates that erroneous positive coarse targets are not the main source of the remaining transient overhead. Bands are paired-bootstrap 95\% intervals.}
\label{fig:app-below-threshold}
\end{figure}

At $T=614400$, the forced-zero design stream spends $4346$ cost units and retains the fixed 40-unit estimation-coarse initialization. The corresponding endpoint fine-count inflation proxy is about $0.72\%$, while cumulative overhead is about $4.3$--$4.4\%$. Earlier opportunity costs were larger and remain integrated into cumulative risk, and finite-sample estimation effects can persist. These accounting quantities are diagnostics, not a causal decomposition.

\subsection{Distributional Illustration of Annotation Gain}
\label{app:gain-illustration}

Finally, we illustrate the magnitude of the static gain over heterogeneous aggregation vectors. Fix $d=20,K=5$ and draw $500{,}000$ independent weights from $\operatorname{Dirichlet}(1,\ldots,1)$. For each draw, set
$\rho=c_{\mathrm F}\sigma_{\mathrm F}^2/(c_{\mathrm C}\sigma_{\mathrm C}^2)$ and evaluate the closed-form gain at $\lambda=\rho\|w\|_2^2$. This is Monte Carlo integration of a theorem quantity: no covariates or labels are generated and no estimator is fitted. At $\rho=20$, approximately half of the sampled weights have positive gain but the unconditional mean gain is only $0.15\%$; the mean rises to approximately $1.73\%$, $5.39\%$, and $7.97\%$ at $\rho=40,100,200$.

\begin{figure}[t]
\centering
\begin{minipage}[t]{0.48\textwidth}
\centering
\includegraphics[width=\linewidth]{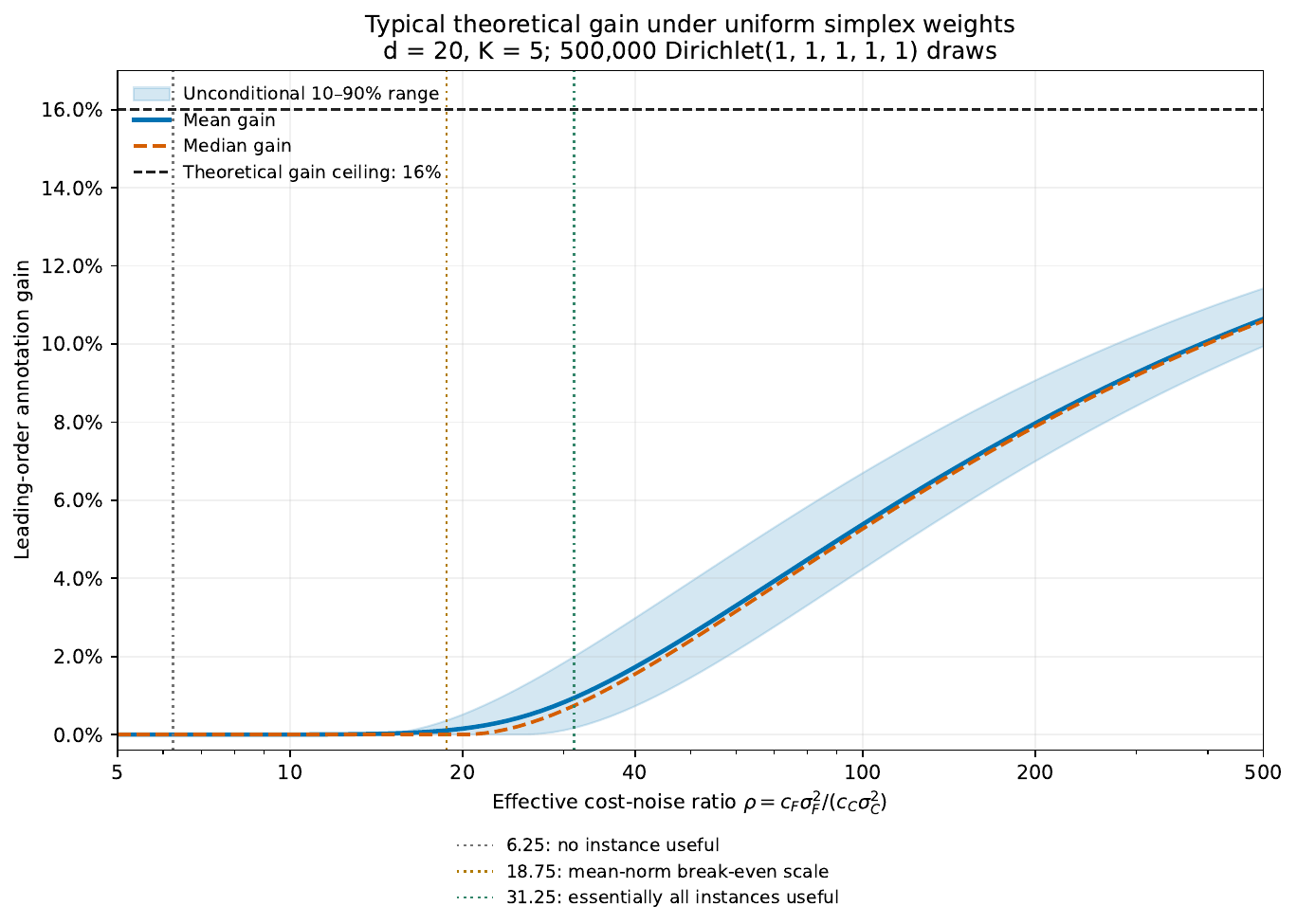}\\[-0.5em]
{\small (a) Distribution of annotation gain}
\end{minipage}
\hfill
\begin{minipage}[t]{0.48\textwidth}
\centering
\includegraphics[width=\linewidth]{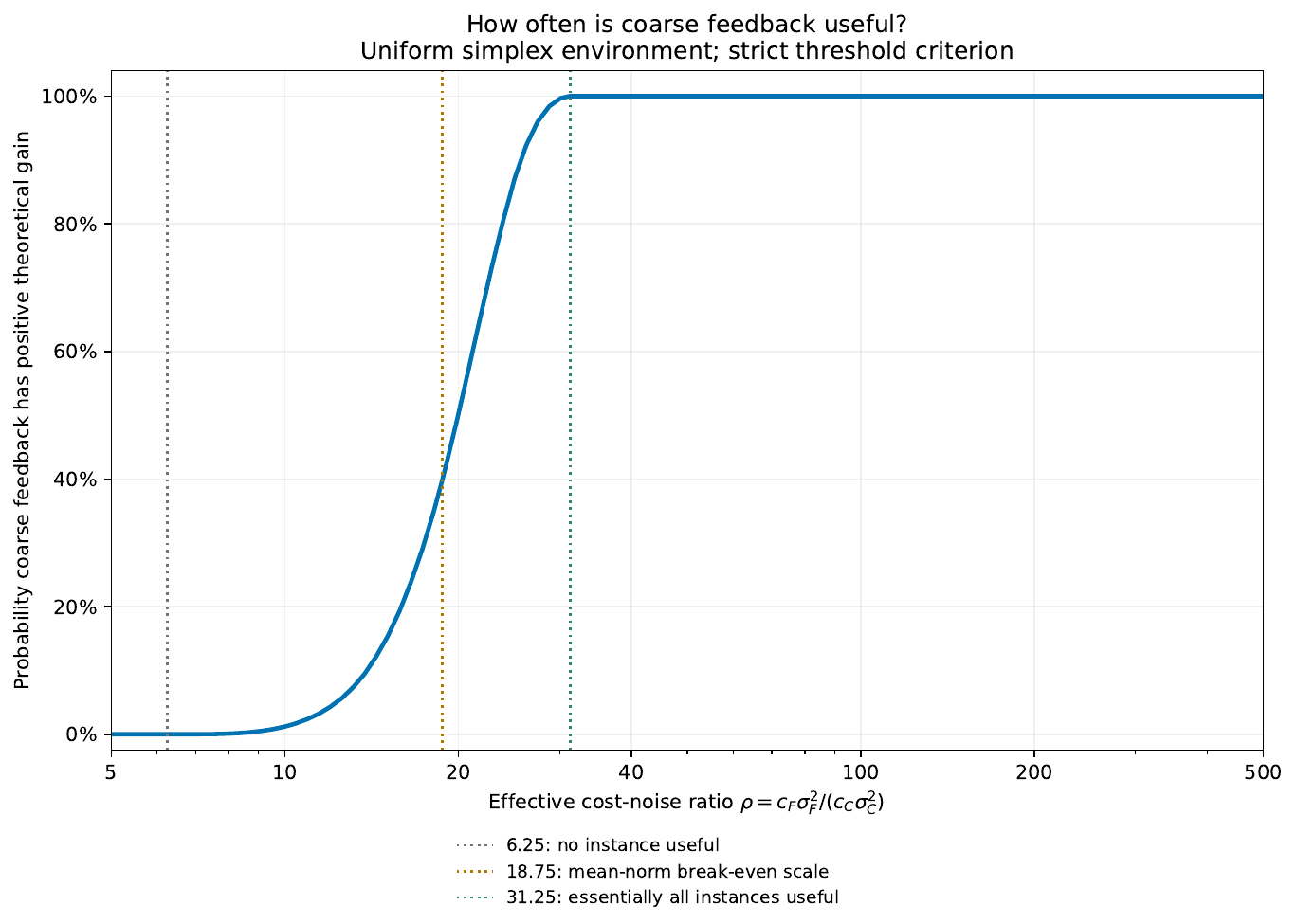}\\[-0.5em]
{\small (b) Probability of positive gain}
\end{minipage}
\caption{\textbf{Theorem-magnitude illustration under random simplex weights.} Positive gain can become common before its magnitude becomes substantial. The variation across sampled weights is heterogeneity under the illustrative Dirichlet environment, not estimation uncertainty.}
\label{fig:app-typical-gain}
\end{figure}

\section{Proof Details}
\label{app:proofs}

This section gives complete proofs of the formal results in \Cref{sec:static,sec:online}. We work in whitened coordinates throughout. Let $w_0:=K^{-1}\mathbf 1$, let $Q\in\mathbb R^{K\times(K-1)}$ have orthonormal columns spanning $\mathbf 1^\perp$, and write $w(v)=w_0+Qv$. We use column stacking $\beta:=\operatorname{vec}(\Theta)$. Constants may depend on the fixed dimensions, costs, variances, $L_x$, $\tau$, $\kappa_\Theta$, and $M_\Theta$, but not on the running budget. Unless otherwise stated, these constants are uniform over the regular parameter class in \Cref{ass:covariates,ass:noise-cost,ass:identifiability}.

\subsection{Proof of \Cref{prop:known-static-risk}}
\label{app:proof-known-risk}

For one fine observation, the derivative of the conditional mean with respect to $\beta$ is the usual multivariate-regression design. Since $\E[xx^\top]=I_d$, its expected information is $\sigma_{\mathrm F}^{-2}I_{dK}$. A coarse observation has mean
$x^\top\Theta w_\star=(w_\star^\top\otimes x^\top)\beta$, so its expected information is
\begin{equation}
\sigma_{\mathrm C}^{-2}
(w_\star w_\star^\top\otimes I_d).
\label{eq:app-known-one-coarse}
\end{equation}
Summing $N_{\mathrm F}$ fine and $N_{\mathrm C}$ coarse observations gives \Cref{eq:known-information}.

Let $\bar w_\star=w_\star/\|w_\star\|_2$. Any perturbation $H\in\mathbb R^{d\times K}$ has the unique orthogonal decomposition
\begin{equation}
H=h\bar w_\star^\top+H_0,
\qquad
H_0\bar w_\star=0,
\label{eq:app-known-decomp}
\end{equation}
with $\|H\|_F^2=\|h\|_2^2+\|H_0\|_F^2$. The $H_0$ subspace has dimension $d(K-1)$ and receives no coarse information because $H_0w_\star=0$. Its information eigenvalue is therefore
$\alpha:=N_{\mathrm F}/\sigma_{\mathrm F}^2$. On the complementary $d$-dimensional subspace, the eigenvalue is
$\alpha+\gamma\|w_\star\|_2^2$, where $\gamma:=N_{\mathrm C}/\sigma_{\mathrm C}^2$. Taking the inverse trace gives \Cref{eq:known-static-risk}.

For completeness, consider an exogenous static sequence with $N_{\mathrm F}\asymp B$ and either $N_{\mathrm C}=0$ or $N_{\mathrm C}\asymp B$, which covers the cost-optimal static designs below. Conditional on the sampled covariates, the Gaussian weighted least-squares estimator has information
\begin{equation}
J_B
=
\frac{1}{\sigma_{\mathrm F}^2}(I_K\otimes S_{\mathrm F})
+
\frac{1}{\sigma_{\mathrm C}^2}(w_\star w_\star^\top\otimes S_{\mathrm C}),
\qquad
S_a:=\sum_{i:A_i=a}x_ix_i^\top,
\end{equation}
and conditional covariance $J_B^{-1}$. Its expectation is exactly $\mathcal I_{\mathrm{kn}}$. Bounded isotropic covariates give exponentially small tails for $\|S_a/N_a-I_d\|_{\mathrm{op}}$. On the event where both relevant Gram matrices lie within one half of their expectations, write $J_B=\mathcal I_{\mathrm{kn}}+\Delta_B$. Since $\|\mathcal I_{\mathrm{kn}}^{-1}\|_{\mathrm{op}}=O(B^{-1})$, $\E\Delta_B=0$, and $\E\|\Delta_B\|_{\mathrm{op}}^2=O(B)$, the resolvent identity
\begin{equation}
J_B^{-1}
=
\mathcal I_{\mathrm{kn}}^{-1}
-
\mathcal I_{\mathrm{kn}}^{-1}\Delta_B\mathcal I_{\mathrm{kn}}^{-1}
+
\mathcal I_{\mathrm{kn}}^{-1}\Delta_B\mathcal I_{\mathrm{kn}}^{-1}\Delta_BJ_B^{-1}
\end{equation}
shows after taking traces and expectations that
$\E\operatorname{tr}(J_B^{-1})=\operatorname{tr}(\mathcal I_{\mathrm{kn}}^{-1})+O(B^{-2})$; the bad-Gram event contributes exponentially little after bounded projection. Thus the projected weighted estimator has risk at most the displayed trace plus $O(B^{-2})$. At the all-fine boundary the same argument reduces to projected multivariate least squares. The fixed-dimensional Gaussian experiment is differentiable in quadratic mean and locally asymptotically normal. The local asymptotic minimax theorem for bowl-shaped loss therefore gives the inverse-information trace as the local lower bound \citep[Theorem~8.11]{van2000asymptotic}. Projection onto the fixed operator-norm ball bounds squared loss, so the corresponding truncated-loss argument and uniform integrability justify the expected-risk statement. Together these facts prove the first-order attainment statement in \Cref{prop:known-static-risk}.

\subsection{Proof of \Cref{thm:known-static-optimum}}
\label{app:proof-known-optimum}

Substituting $N_{\mathrm F}=(1-\eta)B/c_{\mathrm F}$ and $N_{\mathrm C}=\eta B/c_{\mathrm C}$ into \Cref{eq:known-static-risk} gives \Cref{eq:known-budget-risk}. Its first two derivatives are
\begin{align}
\partial_\eta\psi_{\mathrm K}(\eta,\lambda)
&=
\frac{K-1}{(1-\eta)^2}
-
\frac{\lambda-1}{\{1+(\lambda-1)\eta\}^2},
\label{eq:app-known-first}\\
\partial_{\eta\eta}^2\psi_{\mathrm K}(\eta,\lambda)
&=
\frac{2(K-1)}{(1-\eta)^3}
+
\frac{2(\lambda-1)^2}{\{1+(\lambda-1)\eta\}^3}>0.
\label{eq:app-known-second}
\end{align}
Hence the objective is strictly convex. At zero,
$\partial_\eta\psi_{\mathrm K}(0,\lambda)=K-\lambda$. If $\lambda\le K$, strict convexity makes $\eta=0$ the unique minimizer. If $\lambda>K$, the minimizer is interior and solves
\begin{equation}
\frac{\sqrt{K-1}}{1-\eta}
=
\frac{\sqrt{\lambda-1}}{1+(\lambda-1)\eta}.
\end{equation}
Setting $q_{\mathrm K}=\sqrt{(\lambda-1)/(K-1)}$ and rearranging gives \Cref{eq:known-eta-star}.

Substitution yields
\begin{equation}
\psi_{\mathrm K}^\star(\lambda)
=
\frac{\left(\sqrt{(K-1)(\lambda-1)}+1\right)^2}{\lambda}
\end{equation}
above threshold. Since this tends to $K-1$ as $\lambda\to\infty$, \Cref{eq:known-gain-ceiling} follows directly.

\subsection{Proof of \Cref{thm:unknown-frontier}}
\label{app:proof-unknown-frontier}

We first invoke \Cref{lem:unknown-information-split}, proved in \Cref{app:proof-information-split}. It implies that for fixed counts the efficient inverse-information trace for $\Theta_\star$ is
\begin{equation}
\mathcal R_{\mathrm U}(N_{\mathrm F},N_{\mathrm C})
:=
\frac{A_{\mathrm U}\sigma_{\mathrm F}^2}{N_{\mathrm F}}
+
\frac{D_{\mathrm U}}
{N_{\mathrm F}/\sigma_{\mathrm F}^2+
N_{\mathrm C}\|w_\star\|_2^2/\sigma_{\mathrm C}^2}.
\label{eq:app-unknown-fixed-risk}
\end{equation}
For fixed positive proportions, the joint Gaussian experiment in the local coordinate $(\beta,v)$ is differentiable in quadratic mean and locally asymptotically normal. The local asymptotic minimax theorem \citep[Theorem~8.11]{van2000asymptotic} gives the trace of the inverse efficient information as the local lower bound for the $\beta$ target, yielding \Cref{eq:app-unknown-fixed-risk}. The projected two-fold one-step construction analyzed in \Cref{app:proof-conditional-efficient} attains the same leading covariance under an exogenous static schedule; when $N_{\mathrm C}=0$, projected fine OLS gives the boundary case. The final projection onto the fixed operator-norm ball makes the squared loss uniformly bounded, so the asymptotic linear expansion also yields convergence of expected squared loss, not only distributional convergence.

At budget $B$ and coarse spending share $\eta$, substitute
$N_{\mathrm F}=(1-\eta)B/c_{\mathrm F}$ and $N_{\mathrm C}=\eta B/c_{\mathrm C}$ into \Cref{eq:app-unknown-fixed-risk}. With \Cref{eq:known-lambda},
\begin{equation}
\mathcal R_{\mathrm U}(B,\eta)
=
\frac{c_{\mathrm F}\sigma_{\mathrm F}^2d}{B}
\left[
\frac{A_{\mathrm U}/d}{1-\eta}
+
\frac{D_{\mathrm U}/d}{1+(\lambda-1)\eta}
\right]
+o(B^{-1}),
\end{equation}
which is \Cref{eq:unknown-leading-risk}.

To optimize, set $A=A_{\mathrm U}$ and $D=D_{\mathrm U}$ and ignore the common factor $1/d$. Define
\begin{equation}
\phi(\eta,\lambda)
=
\frac{A}{1-\eta}
+
\frac{D}{1+(\lambda-1)\eta}.
\end{equation}
Its derivatives are
\begin{align}
\partial_\eta\phi
&=
\frac{A}{(1-\eta)^2}
-
\frac{D(\lambda-1)}{\{1+(\lambda-1)\eta\}^2},\\
\partial_{\eta\eta}^2\phi
&=
\frac{2A}{(1-\eta)^3}
+
\frac{2D(\lambda-1)^2}{\{1+(\lambda-1)\eta\}^3}>0.
\end{align}
At zero, $\partial_\eta\phi(0,\lambda)=A-D(\lambda-1)$, so the boundary changes sign at
\begin{equation}
1+\frac{A}{D}
=
\frac{A+D}{D}
=
\frac{dK}{d-K+1}.
\end{equation}
Above this boundary, solving the first-order condition gives
\begin{equation}
\eta_{\mathrm U}^\star(\lambda)
=
\frac{q_{\mathrm U}-1}{\lambda-1+q_{\mathrm U}},
\qquad
q_{\mathrm U}:=\sqrt{\frac{D_{\mathrm U}(\lambda-1)}{A_{\mathrm U}}}.
\label{eq:app-unknown-eta}
\end{equation}
This proves the threshold and optimizer.

The all-fine coefficient in the unnormalized objective is $A+D=dK$. Above threshold the exact gap is
\begin{equation}
dK-\phi^\star(\lambda)
=
\frac{A_{\mathrm U}(q_{\mathrm U}-1)^2}{\lambda},
\end{equation}
so
\begin{equation}
G_{\mathrm U}(\lambda)
=
\frac{A_{\mathrm U}(q_{\mathrm U}-1)^2}{\lambda dK}.
\end{equation}
The envelope theorem gives
\begin{equation}
\frac{d}{d\lambda}\phi^\star(\lambda)
=
-\frac{D_{\mathrm U}\eta_{\mathrm U}^\star(\lambda)}
{\{1+(\lambda-1)\eta_{\mathrm U}^\star(\lambda)\}^2}<0
\end{equation}
above threshold, proving that the gain is strictly increasing there. Finally $\phi(\eta,\lambda)\ge A_{\mathrm U}$ for every feasible $\eta$, while $\phi^\star(\lambda)\to A_{\mathrm U}$ as $\lambda\to\infty$. Hence the gain ceiling is
$1-A_{\mathrm U}/(dK)=D_{\mathrm U}/(dK)=(d-K+1)/(dK)$.

\subsection{Proof of \Cref{lem:unknown-information-split}}
\label{app:proof-information-split}

The simplex tangent at the interior point $w_\star$ is $\mathbf1^\perp$. Consider a local perturbation $(H,h_w)$ with $H\in\mathbb R^{d\times K}$ and $h_w\in\mathbf1^\perp$. A fine observation has local mean perturbation $H^\top x$, while a coarse observation has local mean perturbation $x^\top(Hw_\star+\Theta_\star h_w)$. Therefore the expected Fisher-information quadratic form for fixed counts is
\begin{equation}
\mathcal I(H,h_w)
=
\alpha\|H\|_F^2
+
\gamma\|Hw_\star+\Theta_\star h_w\|_2^2,
\qquad
\alpha:=\frac{N_{\mathrm F}}{\sigma_{\mathrm F}^2},
\quad
\gamma:=\frac{N_{\mathrm C}}{\sigma_{\mathrm C}^2}.
\label{eq:app-joint-info-form}
\end{equation}
Let $\mathcal S_\star:=\operatorname{col}(\Theta_\star Q)=\Theta_\star\mathbf1^\perp$. By \Cref{ass:identifiability}, $\Theta_\star Q$ has full column rank, so $\dim(\mathcal S_\star)=K-1$. Profiling \Cref{eq:app-joint-info-form} over $h_w$ gives
\begin{equation}
\mathcal I_{\mathrm{eff}}(H)
=
\alpha\|H\|_F^2
+
\gamma\left\|P_{\mathcal S_\star^\perp}Hw_\star\right\|_2^2.
\label{eq:app-profiled-info}
\end{equation}
Indeed, the second term is the squared distance from $-Hw_\star$ to the nuisance image $\mathcal S_\star$.

Let $\bar w_\star=w_\star/\|w_\star\|_2$ and write
\begin{equation}
H=h\bar w_\star^\top+H_0,
\qquad H_0\bar w_\star=0.
\end{equation}
The $H_0$ subspace has dimension $d(K-1)$ and receives only fine information $\alpha$. Decompose $h=h_\parallel+h_\perp$ with
$h_\parallel\in\mathcal S_\star$ and $h_\perp\in\mathcal S_\star^\perp$. The $K-1$ directions in $\mathcal S_\star$ can be canceled by a nuisance perturbation $h_w$, so they also receive only $\alpha$. The remaining $d-K+1$ directions receive $\alpha+\gamma\|w_\star\|_2^2$. Thus the fine-only multiplicity is
$d(K-1)+(K-1)=(K-1)(d+1)=A_{\mathrm U}$ and the fine-plus-coarse multiplicity is $d-K+1=D_{\mathrm U}$.

For later use, write $B_\Theta:=\Theta Q$. In the joint coordinate $(\beta,v)$ the expected information blocks are
\begin{align}
I_{\beta\beta}
&=
\alpha I_{dK}+\gamma(ww^\top\otimes I_d),
\label{eq:app-Ibb}\\
I_{\beta v}
&=
\gamma(w\otimes B_\Theta),
\label{eq:app-Ibv}\\
I_{vv}
&=
\gamma B_\Theta^\top B_\Theta.
\label{eq:app-Ivv}
\end{align}
Their $\beta$ Schur complement is
\begin{equation}
\Sigma_\beta(\Theta,w)
=
\alpha I_{dK}
+
\gamma\bigl(ww^\top\otimes(I_d-P_{\Theta Q})\bigr),
\label{eq:app-beta-schur}
\end{equation}
where
$P_{\Theta Q}:=\Theta Q(Q^\top\Theta^\top\Theta Q)^{-1}Q^\top\Theta^\top$.
At the truth, the inverse trace of \Cref{eq:app-beta-schur} is exactly \Cref{eq:app-unknown-fixed-risk}.

\subsection{Proof of \Cref{cor:static-benchmark-attainment}}
\label{app:proof-static-benchmark}

For $\mathsf B_1$, choose the static fraction $\eta_{\mathrm K}^\star$ from \Cref{thm:known-static-optimum} and use known-$w_\star$ weighted Gaussian least squares. By \Cref{app:proof-known-risk}, at budget $b$ its expected risk is
\begin{equation}
L_b^{\mathsf B_1}
\le
\frac{\Phi_{\mathrm K}^\star}{b}+O(b^{-2}),
\end{equation}
including bounded count rounding. Summing gives
$\mathfrak R_T^{\mathsf B_1}\le\Phi_{\mathrm K}^\star\log T+O(1)$.

For $\mathsf B_2$, choose the exogenous unknown-aggregation fraction $\eta_{\mathrm U}^\star$. If $\eta_{\mathrm U}^\star=0$, projected fine OLS again has an $O(b^{-2})$ remainder. If $\eta_{\mathrm U}^\star>0$, both fine and coarse counts are linear in $b$ and the static two-fold specialization of \Cref{lem:conditional-efficient-estimation} gives
\begin{equation}
L_b^{\mathsf B_2}
\le
\frac{\Phi_{\mathrm U}^\star}{b}+O(b^{-3/2}).
\end{equation}
The remainder is summable, yielding the stated $\mathsf B_2$ cumulative bound. The local minimax interpretations follow from the known-weight inverse-information argument and the unknown-weight local lower bound in \Cref{app:proof-unknown-frontier}.

\subsection{Proof of \Cref{thm:online-main}}
\label{app:proof-online-main}

Let $\Xi(b)$ be design-stream spending and $b_{\mathrm E}=b-\Xi(b)$. By \Cref{lem:pilot-tracking},
\begin{equation}
\Xi(b)=O\!\left(\frac{b}{(\log b)^2}\right),
\end{equation}
so $b_{\mathrm E}\ge b/2$ for all sufficiently large $b$ and
\begin{equation}
0\le
\frac1{b_{\mathrm E}}-\frac1b
=
\frac{\Xi(b)}{bb_{\mathrm E}}
=
O\!\left(\frac1{b(\log b)^2}\right).
\label{eq:app-design-opportunity}
\end{equation}
The conditional-efficiency lemma is stated for all sufficiently large estimation endpoints. Let $b_\star$ be a fixed threshold after which it applies. The finitely many levels $B_0\le b<b_\star$ contribute only $O(1)$ to cumulative risk because projection onto the compact parameter set uniformly bounds $L_b$.

At a sufficiently large estimation endpoint, taking total expectation in \Cref{eq:conditional-efficient-risk} gives
\begin{equation}
L_b
\le
\frac{c_{\mathrm F}\sigma_{\mathrm F}^2d}{b_{\mathrm E}}
\E\psi_{\mathrm U}(\eta_b^{\mathrm E},\lambda)
+
O(b_{\mathrm E}^{-3/2}).
\end{equation}
Subtract
$\Phi_{\mathrm U}^\star/b=
(c_{\mathrm F}\sigma_{\mathrm F}^2d/b)
\psi_{\mathrm U}(\eta_{\mathrm U}^\star,\lambda)$,
apply \Cref{lem:pilot-tracking}, and use boundedness of $\psi_{\mathrm U}$ on the compact range of realized shares. This gives \Cref{eq:per-budget-online-excess} at estimation endpoints.

It remains to cover the budget levels at which the predictor is held fixed. Within an estimation block, consecutive estimation endpoints are separated by at most the fixed maximum action cost, so repeating an endpoint bound across these bounded gaps preserves summability. For a design block in epoch $m$, let $s_m$ be the last estimation endpoint before the block and let $D_m$ be its budget length. Bounded endpoint residue gives $s_m\asymp B_m$, while $D_m=O(B_m/m^2)$. The predictor is fixed throughout the block, so its leading contribution is $D_m\Phi_{\mathrm U}^\star/s_m$. Relative to the harmonic benchmark on the same budget levels,
\begin{equation}
D_m\frac{\Phi_{\mathrm U}^\star}{s_m}
-\Phi_{\mathrm U}^\star\sum_{j=1}^{D_m}\frac1{s_m+j}
=O\!\left(\frac{D_m^2}{s_m^2}\right)
=O(m^{-4}).
\label{eq:app-held-design-leading}
\end{equation}
Repeating the three remainder terms in \Cref{eq:per-budget-online-excess} over the same block contributes respectively $O(m^{-4})$, $O(m2^{-m})$, and $O(2^{-m/2}/m^2)$. These sequences are summable over epochs. Thus the non-endpoint budget levels contribute only $O(1)$ additional cumulative excess, and
\begin{equation}
\mathfrak R_T^{\mathrm{alg}}
\le
\Phi_{\mathrm U}^\star\sum_{b=B_0}^T\frac1b+O(1)
=
\Phi_{\mathrm U}^\star\log T+O(1).
\end{equation}

All constants in \Cref{lem:pilot-tracking,lem:conditional-efficient-estimation} depend only on the fixed regularity bounds. Therefore the same per-budget inequality holds uniformly on any sufficiently small $\mathcal U_\delta$, with $\Phi_{\mathrm U}^\star$ evaluated at the local parameter. The map $\vartheta\mapsto\Phi_{\mathrm U}^\star(\vartheta)$ is continuous on the regular class: $w\mapsto\|w\|_2^2$ is continuous, the minimizer remains in a compact interval away from one, and the minimum of the continuous objective is continuous. Consequently
\begin{equation}
\sup_{\vartheta\in\mathcal U_\delta}\mathfrak R_T^{\mathrm{alg}}(\vartheta)
\le
\bigl(\Phi_{\mathrm U}^\star(\vartheta_\star)+o_\delta(1)\bigr)\log T+O_\delta(1),
\end{equation}
proving the locally uniform statement in \Cref{thm:online-main}.

The lower bound follows from \Cref{lem:adaptive-information-lower}: sum \Cref{eq:adaptive-information-lower} over $b$, use that the Bayes risk is at most the supremum risk on $\mathcal U_\delta$, divide by $\log T$, and let $T\to\infty$ and then $\delta\downarrow0$. This proves \Cref{eq:online-lower}.

Finally, the pointwise objectives satisfy
\begin{equation}
\psi_{\mathrm U}(\eta,\lambda)-\psi_{\mathrm K}(\eta,\lambda)
=
\frac{K-1}{d}
\left[
\frac1{1-\eta}
-
\frac1{1+(\lambda-1)\eta}
\right]
\ge0,
\end{equation}
because $1+(\lambda-1)\eta-(1-\eta)=\lambda\eta\ge0$. Hence $\Phi_{\mathrm U}^\star\ge\Phi_{\mathrm K}^\star$, and the stated known-versus-unknown decomposition is the identity $\Phi_{\mathrm U}^\star=\Phi_{\mathrm K}^\star+(\Phi_{\mathrm U}^\star-\Phi_{\mathrm K}^\star)$ inserted into the upper bound.

\subsection{Proof of \Cref{lem:pilot-tracking}}
\label{app:proof-pilot-tracking}

We use a guarded projected regression primitive repeatedly. Given $n$ whitened observations from a linear regression $Y=B_\star^\top x+\varepsilon$, let
$S_n=\sum_{i=1}^n x_ix_i^\top$. If $\lambda_{\min}(S_n)\ge n/2$, use ordinary least squares and project the coefficient onto a relevant fixed closed convex compact set containing the truth; otherwise return any fixed point in that set. Bounded isotropic covariates imply
\begin{equation}
\Pr\!\left(\|S_n/n-I_d\|_{\mathrm{op}}>1/2\right)
\le C e^{-cn}.
\label{eq:app-gram-tail}
\end{equation}
Conditionally on the covariates and on the good event, the Gaussian regression error has second and fourth moments of orders $n^{-1}$ and $n^{-2}$. Projection is nonexpansive and bounds the error on the bad event. Thus, for $p=1,2$,
\begin{equation}
\E\|\widetilde B-B_\star\|_F^{2p}
\le C_p(n\vee1)^{-p}.
\label{eq:app-guarded-regression}
\end{equation}
This proves the moment statement used below and also makes every pilot well defined on every sample path.

\paragraph{Design-stream size and aggregation pilot.}
At the start of epoch $m$, the installed pilot is based on the fixed initialization and design blocks $j<m$; the design block collected in epoch $m$ first enters the pilot used in epoch $m+1$.
Let $B_m=2^mB_0$ and $\Delta B_m=B_{m+1}-B_m$. Since
$\zeta_m=\zeta_0/(m+2)^2$,
\begin{equation}
\Xi(B_m)
\asymp
\sum_{j<m}\frac{2^j}{(j+2)^2}
=
\Theta\!\left(\frac{2^m}{m^2}\right).
\label{eq:app-design-sum}
\end{equation}
The upper bound follows by splitting the sum at $m/2$; the lower bound follows from the last fixed number of terms. The fixed design split is bounded away from zero and one, so both fine and coarse design counts are $\Theta(2^m/m^2)$. Applying \Cref{eq:app-guarded-regression} gives, for $p=1,2$,
\begin{align}
\E\|\widetilde\Theta_m-\Theta_\star\|_F^{2p}
&\le C\left(\frac{m^2}{2^m}\right)^p,\\
\E\|\widetilde u_m-u_\star\|_2^{2p}
&\le C\left(\frac{m^2}{2^m}\right)^p.
\label{eq:app-design-regression-rates}
\end{align}

The constrained pilot satisfies
\begin{equation}
\widetilde w_m
\in
\argmin_{w\in\Delta_K,\;w_k\ge\tau}
\|\widetilde u_m-\widetilde\Theta_mw\|_2^2.
\end{equation}
The truth is feasible because $w_{\star,k}\ge2\tau$. On the event
$\|(\widetilde\Theta_m-\Theta_\star)Q\|_{\mathrm{op}}\le\kappa_\Theta/2$,
Weyl's inequality gives
$\sigma_{\min}(\widetilde\Theta_mQ)\ge\kappa_\Theta/2$.
Because $\widetilde w_m-w_\star\in\mathbf1^\perp$, write
$\widetilde w_m-w_\star=Qh_m$. Optimality and feasibility imply
\begin{align}
\frac{\kappa_\Theta}{2}\|\widetilde w_m-w_\star\|_2
&\le
\|\widetilde\Theta_m(\widetilde w_m-w_\star)\|_2\\
&\le
2\|\widetilde u_m-u_\star\|_2
+2\|\widetilde\Theta_m-\Theta_\star\|_{\mathrm{op}}\|w_\star\|_2.
\end{align}
The complementary event has exponentially small probability by \Cref{eq:app-gram-tail} together with the conditional Gaussian tail for the guarded regression, while compact simplex constraints bound the error there. Hence
\begin{equation}
\E\|\widetilde w_m-w_\star\|_2^{2p}
\le C\left(\frac{m^2}{2^m}\right)^p.
\label{eq:app-w-pilot}
\end{equation}
Let $\rho=c_{\mathrm F}\sigma_{\mathrm F}^2/(c_{\mathrm C}\sigma_{\mathrm C}^2)$. Since simplex vectors have norm at most one,
\begin{equation}
|\widetilde\lambda_m-\lambda|
=
\rho\left|\|\widetilde w_m\|_2^2-\|w_\star\|_2^2\right|
\le2\rho\|\widetilde w_m-w_\star\|_2,
\end{equation}
so
\begin{equation}
\E|\widetilde\lambda_m-\lambda|^2
\le C\frac{m^2}{2^m}.
\label{eq:app-lambda-pilot}
\end{equation}
Because both $w_\star$ and $\widetilde w_m$ lie in the simplex,
$\lambda,\widetilde\lambda_m\in[\rho/K,\rho]$.

\paragraph{Quadratic stability of the plug-in optimum.}
Every minimizer of $\psi_{\mathrm U}$ lies in a fixed compact interval away from one: comparison with $\eta=0$ gives
$(A_{\mathrm U}/d)/(1-\eta_{\mathrm U}^\star)\le K$, hence
$\eta_{\mathrm U}^\star\le D_{\mathrm U}/(dK)<1$.
On a slightly larger compact interval and for
$\lambda\in[\rho/K,\rho]$, there are constants $\mu>0$ and $C_{\eta\lambda}<\infty$ such that
\begin{equation}
\partial_{\eta\eta}^2\psi_{\mathrm U}\ge\mu,
\qquad
|\partial_{\eta\lambda}^2\psi_{\mathrm U}|\le C_{\eta\lambda}.
\end{equation}
Let $\eta_i=\eta_{\mathrm U}^\star(\lambda_i)$. The variational inequalities for constrained minimizers are
\begin{equation}
\partial_\eta\psi_{\mathrm U}(\eta_1,\lambda_1)(\eta_2-\eta_1)\ge0,
\qquad
\partial_\eta\psi_{\mathrm U}(\eta_2,\lambda_2)(\eta_1-\eta_2)\ge0.
\end{equation}
Adding them and using strong monotonicity in $\eta$ gives
\begin{equation}
|\eta_2-\eta_1|
\le
\frac{C_{\eta\lambda}}{\mu}|\lambda_2-\lambda_1|.
\label{eq:app-eta-lipschitz}
\end{equation}
This argument also covers a minimizer at the boundary zero. Next, optimality of $\eta_2$ under $\lambda_2$ implies
$\psi_{\mathrm U}(\eta_2,\lambda_2)\le\psi_{\mathrm U}(\eta_1,\lambda_2)$. Therefore
\begin{align}
&\psi_{\mathrm U}(\eta_2,\lambda_1)-\psi_{\mathrm U}(\eta_1,\lambda_1)\\
&\quad\le
\bigl[\psi_{\mathrm U}(\eta_2,\lambda_1)-\psi_{\mathrm U}(\eta_2,\lambda_2)\bigr]
-
\bigl[\psi_{\mathrm U}(\eta_1,\lambda_1)-\psi_{\mathrm U}(\eta_1,\lambda_2)\bigr]\\
&\quad\le
C_{\eta\lambda}|\eta_2-\eta_1|\,|\lambda_2-\lambda_1|
\le
\frac{C_{\eta\lambda}^2}{\mu}|\lambda_2-\lambda_1|^2.
\label{eq:app-quadratic-plugin}
\end{align}
Applying this with $\lambda_1=\lambda$ and $\lambda_2=\widetilde\lambda_m$, then using \Cref{eq:app-lambda-pilot}, yields
\begin{equation}
\E\bigl[
\psi_{\mathrm U}(\widehat\eta_m,\lambda)
-
\psi_{\mathrm U}(\eta_{\mathrm U}^\star,\lambda)
\bigr]
\le C\frac{m^2}{2^m}.
\label{eq:app-epoch-objective-gap}
\end{equation}
No margin from the break-even boundary is used.

\paragraph{Accumulated tracking.}
Ignoring bounded endpoint residues, the cumulative estimation-stream coarse fraction at a budget $b\in[B_m,B_{m+1})$ is a cost-weighted convex combination of the installed epoch targets,
\begin{equation}
\bar\eta_b^{\mathrm E}
=
\sum_{j=0}^m\omega_{j,b}\widehat\eta_j,
\qquad
\omega_{j,b}\ge0,
\quad
\sum_j\omega_{j,b}=1,
\end{equation}
with $\omega_{j,b}\le C2^{j-m}$ for $j<m$. Convexity and \Cref{eq:app-epoch-objective-gap} give
\begin{align}
\E\bigl[\psi_{\mathrm U}(\bar\eta_b^{\mathrm E},\lambda)-
\psi_{\mathrm U}(\eta_{\mathrm U}^\star,\lambda)\bigr]
&\le
C2^{-m}
+C\sum_{j=1}^m2^{j-m}\frac{j^2}{2^j}\\
&\le
C\frac{m^3}{2^m}.
\end{align}
Writing $s$ as coarse expenditure minus its target expenditure, a fine action changes $s$ by $-\widehat\eta_m c_{\rm F}$ and a coarse action by $(1-\widehat\eta_m)c_{\rm C}$. The smaller-absolute-value rule therefore keeps $|s|\le \max\{c_{\rm F},c_{\rm C}\}$ by induction, up to the bounded terminal affordability residue.
The signed-imbalance tracker changes cumulative coarse expenditure by only $O(1)$ within each epoch; bounded unused residues contribute $O(m)$ through epoch $m$. Since $\partial_\eta\psi_{\mathrm U}$ is bounded on the compact interval containing all realized shares, these effects add only $O(m/2^m)$. Using $2^m\asymp b$ and $m\asymp\log b$ gives the second statement of \Cref{eq:pilot-tracking-rate}; \Cref{eq:app-design-sum} gives the first.

\subsection{Proof of \Cref{lem:conditional-efficient-estimation}}
\label{app:proof-conditional-efficient}

Fix an epoch and condition on its design history $\mathcal G_m$. The target share and the entire estimation action schedule through the epoch are then deterministic, as are the parity-fold assignments within each resolution. All estimation covariates and noises remain independent of $\mathcal G_m$. Let $N_{a,r}$ be the number of estimation observations of resolution $a\in\{\mathrm F,\mathrm C\}$ in fold $r\in\{1,2\}$ at budget $b$. The deterministic parity rule gives $|N_{a,1}-N_{a,2}|\le1$.

\paragraph{Opposite-fold pilots.}
For score fold $r$, fit guarded projected fine and coarse regressions on the opposite fold $-r$ using the primitive from \Cref{eq:app-guarded-regression}, obtaining $\widetilde\Theta^{(-r)}$ and $\widetilde u^{(-r)}$. Set
\begin{equation}
\widetilde w^{(-r)}
\in
\argmin_{w\in\Delta_K,\;w_k\ge\tau}
\|\widetilde u^{(-r)}-\widetilde\Theta^{(-r)}w\|_2^2,
\qquad
\widetilde v^{(-r)}=Q^\top(\widetilde w^{(-r)}-w_0).
\end{equation}
The finite initialization guarantees a fixed positive number of coarse observations in each fold; hence the nuisance block $I_{vv}$ is well defined even when the post-initialization coarse count remains bounded. If a coarse Gram matrix is ill conditioned, the guarded regression returns a bounded projected fallback. Because the fine-budget share is uniformly bounded away from zero, $N_{\mathrm F,-r}=\Theta(b_{\mathrm E})$. The same tangent-stability argument as in \Cref{app:proof-pilot-tracking} gives, for $q=1,2$,
\begin{align}
\E\|\widetilde\Theta^{(-r)}-\Theta_\star\|_F^{2q}
&\le Cb_{\mathrm E}^{-q},
\label{eq:app-cf-theta-rate}\\
\E\|\widetilde v^{(-r)}-v_\star\|_2^{2q}
&\le C\left(b_{\mathrm E}^{-q}+(N_{\mathrm C,-r}\vee1)^{-q}\right).
\label{eq:app-cf-v-rate}
\end{align}
The second bound is allowed to remain $O(1)$ when the coarse count is bounded; this is sufficient because the nuisance contribution to the target row of the inverse information is then downweighted by the coarse information itself.

\paragraph{One-step map and inverse-information bounds.}
Let $S_r(\vartheta)$ be the Gaussian score on fold $r$ and let $I_r(\vartheta)$ be its expected Fisher information conditional on the deterministic fold counts. Define
\begin{equation}
\check\vartheta^{(r)}
=
\widetilde\vartheta^{(-r)}
+
I_r(\widetilde\vartheta^{(-r)})^{-1}
S_r(\widetilde\vartheta^{(-r)}).
\label{eq:app-cf-onestep}
\end{equation}
If the fine-pilot tangent guard
$\sigma_{\min}(\widetilde\Theta^{(-r)}Q)\ge\kappa_\Theta/2$ fails, use the bounded projected fine estimator for that fold. The guard failure probability is $O(e^{-cb_{\mathrm E}})$ by \Cref{eq:app-gram-tail} and the conditional Gaussian tail of the fine-regression error. The final estimator averages the two $\beta$ components and projects the reshaped matrix onto the compact operator-norm ball.

For one score fold let
$\alpha_r=N_{\mathrm F,r}/\sigma_{\mathrm F}^2$ and
$\gamma_r=N_{\mathrm C,r}/\sigma_{\mathrm C}^2$. The exact information blocks are \Cref{eq:app-Ibb,eq:app-Ibv,eq:app-Ivv}. Standard block inversion gives
\begin{align}
[I_r^{-1}]_{\beta v}
&=
-\frac1{\alpha_r}
(w\otimes\Theta Q)
\{(\Theta Q)^\top(\Theta Q)\}^{-1},
\label{eq:app-inverse-bv}\\
[I_r^{-1}]_{\beta\beta}
&=
I_{\beta\beta}^{-1}
+
\frac{\gamma_r}{\alpha_r(\alpha_r+\gamma_r\|w\|_2^2)}
\bigl(ww^\top\otimes P_{\Theta Q}\bigr).
\label{eq:app-inverse-bb}
\end{align}
On the regular neighborhood, $\alpha_r\asymp b_{\mathrm E}$ and $0<\gamma_r\le Cb_{\mathrm E}$. Hence the two blocks in the $\beta$ row are $O(b_{\mathrm E}^{-1})$. Moreover, for two parameters $\vartheta_1,\vartheta_2$ in that neighborhood,
\begin{align}
\|[I_r(\vartheta_1)^{-1}]_{\beta\beta}-[I_r(\vartheta_2)^{-1}]_{\beta\beta}\|_{\mathrm{op}}
&\le
C\frac{\gamma_r}{b_{\mathrm E}^2}\|\vartheta_1-\vartheta_2\|,
\label{eq:app-inverse-bb-lipschitz}\\
\|[I_r(\vartheta_1)^{-1}]_{\beta v}-[I_r(\vartheta_2)^{-1}]_{\beta v}\|_{\mathrm{op}}
&\le
\frac{C}{b_{\mathrm E}}\|\vartheta_1-\vartheta_2\|.
\label{eq:app-inverse-bv-lipschitz}
\end{align}
These follow directly from \Cref{eq:app-inverse-bv,eq:app-inverse-bb} because
$\Theta\mapsto P_{\Theta Q}$ and
$\Theta\mapsto\{(\Theta Q)^\top(\Theta Q)\}^{-1}$ are Lipschitz when the tangent singular value is bounded away from zero. The factor $\gamma_r/b_{\mathrm E}^2$ is important when coarse counts are sublinear.

\paragraph{Exact score decomposition.}
Write
$\Delta_\Theta=\widetilde\Theta^{(-r)}-\Theta_\star$,
$\Delta_v=\widetilde v^{(-r)}-v_\star$, and
$B_\star=\Theta_\star Q$. For a coarse score observation, the residual at the pilot is exactly
\begin{equation}
Y_i^{\mathrm C}-x_i^\top\widetilde\Theta^{(-r)}w(\widetilde v^{(-r)})
=
\varepsilon_i^{\mathrm C}-p_i-q_i,
\label{eq:app-coarse-residual}
\end{equation}
where
\begin{equation}
p_i=x_i^\top\Delta_\Theta w_\star+x_i^\top B_\star\Delta_v,
\qquad
q_i=x_i^\top\Delta_\Theta Q\Delta_v.
\end{equation}
The pilot score directions are
$(w_\star+Q\Delta_v)\otimes x_i$ for $\beta$ and
$(B_\star+\Delta_\Theta Q)^\top x_i$ for $v$. Expanding \Cref{eq:app-cf-onestep} around $\vartheta_\star$ gives
\begin{equation}
\check\beta^{(r)}-\beta_\star
=
[I_r(\vartheta_\star)^{-1}S_r(\vartheta_\star)]_\beta
+R_{r,1}+R_{r,2}+R_{r,3}+R_{r,4}+R_{r,5},
\label{eq:app-five-remainders}
\end{equation}
where the five remainder classes are:
\begin{enumerate}
    \item $R_{r,1}$: the quadratic mean term $q_i$ multiplied by the truth score directions;
    \item $R_{r,2}$: the perturbed score directions $Q\Delta_v\otimes x_i$ and $(\Delta_\Theta Q)^\top x_i$, multiplied by $\varepsilon_i^{\mathrm C}-p_i-q_i$;
    \item $R_{r,3}$: the empirical-score Jacobian minus its expected Fisher information, multiplied by $(\Delta_\Theta,\Delta_v)$;
    \item $R_{r,4}$: replacing the truth inverse information by the inverse information evaluated at the pilot;
    \item $R_{r,5}$: the exponentially rare fine-pilot guard failure, handled by the bounded projected fallback.
\end{enumerate}
This list is exhaustive because the fine mean is linear and the coarse mean is exactly bilinear in $(\Theta,v)$. To make the bookkeeping explicit, abbreviate $\widetilde\vartheta=\widetilde\vartheta^{(-r)}$, let
$M_r(\vartheta)=I_r(\vartheta)^{-1}$, and write $M_{r,\beta\bullet}$ for its $\beta$ block row. For a coarse observation define the truth joint score direction
\[
g_i^\star
=
\begin{pmatrix}
w_\star\otimes x_i\\
B_\star^\top x_i
\end{pmatrix},
\qquad
\delta g_i
=
\begin{pmatrix}
Q\Delta_v\otimes x_i\\
(\Delta_\Theta Q)^\top x_i
\end{pmatrix}.
\]
Let $\widehat J_r(\widetilde\vartheta)$ denote the negative empirical score Jacobian on fold $r$. After the linear score term cancels against its expected Jacobian, the first four classes can be represented as
\begin{align}
R_{r,1}
&=
-\bigl[M_r(\widetilde\vartheta)\bigr]_{\beta\bullet}
\frac1{\sigma_{\mathrm C}^2}
\sum_{i\in\mathcal C_r}g_i^\star q_i,\\
R_{r,2}
&=
\bigl[M_r(\widetilde\vartheta)\bigr]_{\beta\bullet}
\frac1{\sigma_{\mathrm C}^2}
\sum_{i\in\mathcal C_r}\delta g_i
(\varepsilon_i^{\mathrm C}-p_i-q_i),\\
R_{r,3}
&=
\bigl[M_r(\widetilde\vartheta)\bigr]_{\beta\bullet}
\{I_r(\widetilde\vartheta)-\widehat J_r(\widetilde\vartheta)\}
(\widetilde\vartheta-\vartheta_\star),\\
R_{r,4}
&=
\bigl[M_r(\widetilde\vartheta)-M_r(\vartheta_\star)\bigr]_{\beta\bullet}
S_r(\vartheta_\star),
\end{align}
with terms that differ only by truth-versus-pilot score directions grouped into the corresponding class. $R_{r,5}$ is supported on the guard-failure event.

We now bound each class conditionally on the opposite-fold pilot. Bounded covariates and Gaussian fourth moments imply the usual $L_2$ bounds for centered sums. Using \Cref{eq:app-cf-theta-rate,eq:app-cf-v-rate}, $\gamma_r\le Cb_{\mathrm E}$, and the target-row inverse bounds above gives
\begin{align}
\|R_{r,1}\|_{L_2}
&\le
C\frac{\gamma_r}{b_{\mathrm E}}
\|\Delta_\Theta\|_{L_4}\|\Delta_v\|_{L_4}
\le \frac{C}{b_{\mathrm E}},\\
\|R_{r,2}\|_{L_2}
&\le
C\left[
\frac{\gamma_r}{b_{\mathrm E}}
(\|\Delta_v\|_{L_4}^2+\|\Delta_\Theta\|_{L_4}\|\Delta_v\|_{L_4})
+
\frac{\sqrt{\gamma_r}}{b_{\mathrm E}}
(\|\Delta_v\|_{L_4}+\|\Delta_\Theta\|_{L_4})
\right]
\le\frac{C}{b_{\mathrm E}},\\
\|R_{r,3}\|_{L_2}
&\le
\frac{C}{b_{\mathrm E}}
\left(
\sqrt{N_{\mathrm F,r}}\|\Delta_\Theta\|_{L_4}
+
\sqrt{N_{\mathrm C,r}}\|\Delta_v\|_{L_4}
\right)
\le\frac{C}{b_{\mathrm E}},\\
\|R_{r,4}\|_{L_2}
&\le\frac{C}{b_{\mathrm E}},
\label{eq:app-remainder-bounds}
\end{align}
where the last inequality uses \Cref{eq:app-inverse-bb-lipschitz,eq:app-inverse-bv-lipschitz} together with truth-score sizes $O_{L_4}(\sqrt{b_{\mathrm E}})$ and $O_{L_4}(\sqrt{\gamma_r})$. The guard contribution satisfies
$\E\|R_{r,5}\|_2^2\le Ce^{-cb_{\mathrm E}}$ because the fallback is bounded. Thus
\begin{equation}
\E\|R_{r,\beta}\|_2^2\le\frac{C}{b_{\mathrm E}^2}.
\label{eq:app-cf-rem-square}
\end{equation}
The bounds remain valid when $N_{\mathrm C,r}$ is bounded: then $\Delta_v$ need not converge, but every occurrence is multiplied by either $\gamma_r/b_{\mathrm E}$ or $\sqrt{\gamma_r}/b_{\mathrm E}$, which is enough for \Cref{eq:app-remainder-bounds}.

\paragraph{Leading variance and risk.}
Conditional on $\mathcal G_m$, the score folds are independent and centered, and
$\operatorname{Cov}(S_r(\vartheta_\star)\mid\mathcal G_m)=I_r(\vartheta_\star)$. Hence the two leading influence functions in \Cref{eq:app-five-remainders} have covariances
$[I_r(\vartheta_\star)^{-1}]_{\beta\beta}$. Their average has mean-square
\begin{equation}
\frac14\sum_{r=1}^2
\operatorname{tr}[I_r(\vartheta_\star)^{-1}]_{\beta\beta}
=
\mathcal R_{\mathrm U}(N_{\mathrm F}^{\mathrm E},N_{\mathrm C}^{\mathrm E})
+O(b_{\mathrm E}^{-2}),
\label{eq:app-cf-leading-variance}
\end{equation}
because each fold has half of the full counts up to $O(1)$, so each fold inverse information is twice the full-count inverse up to $O(b_E^{-2})$; the factor $1/4$ from averaging the two independent fold estimates restores the full-sample efficient trace.
The leading term has $L_2$ norm $O(b_{\mathrm E}^{-1/2})$. Cauchy--Schwarz with \Cref{eq:app-cf-rem-square} therefore makes the influence--remainder cross term $O(b_{\mathrm E}^{-3/2})$, and the squared remainder is $O(b_{\mathrm E}^{-2})$. Final projection cannot increase squared prediction loss. We obtain
\begin{equation}
L_b(\mathcal G_m)
\le
\mathcal R_{\mathrm U}(N_{\mathrm F}^{\mathrm E},N_{\mathrm C}^{\mathrm E})
+
O(b_{\mathrm E}^{-3/2}).
\end{equation}
Writing the count-based efficient trace in terms of $b_{\mathrm E}$ and $\eta_b^{\mathrm E}$ gives \Cref{eq:conditional-efficient-risk}. All constants used above are uniform over the compact regular class and over the predictable schedules generated by EET, proving the lemma.

\subsection{Proof of \Cref{lem:adaptive-information-lower}}
\label{app:proof-adaptive-lower}

Let $\vartheta=(\beta,v)$ and choose a smooth prior $\Pi_\delta$ with compact support strictly inside the regular coordinate ball $\mathcal U_\delta$. Its density is $C^1$, vanishes at the boundary, and has finite Fisher-information matrix $J(\Pi_\delta)$. The Gaussian observation model is differentiable in quadratic mean on this neighborhood, so the multivariate van Trees inequality applies \citep{gill1995applications}. It is essential that the prior vary both $\beta$ and $v$; fixing $v$ would remove the nuisance price.

Fix an arbitrary predictable resolution policy and let $\tau_b$ be the last completed query whose known cumulative cost does not exceed $b$. If $H_{t-1}$ is the observed history before query $t$, the stopped-history density factorizes as
\begin{equation}
p_\vartheta(H_{\tau_b})
=
\prod_{t=1}^{\tau_b}
q_t(A_t\mid H_{t-1})\,p(x_t)\,
p_\vartheta(Y_t^{A_t}\mid x_t,A_t),
\label{eq:app-adaptive-factorization}
\end{equation}
where each predictable policy kernel $q_t$ is parameter free. The stopping rule is also parameter free conditional on the action sequence because costs are known. Therefore the history score is the sum of queried-observation scores. These are martingale differences conditional on the past, so their cross terms have zero expectation and the Fisher information is the expected sum of the queried-label information matrices. Let its prior average be $\overline I_b(\Pi_\delta)$.

The multivariate van Trees inequality for the target map $\vartheta\mapsto\beta$ gives
\begin{equation}
\int
\E_\vartheta
[(\widehat\beta_b-\beta)(\widehat\beta_b-\beta)^\top]
\,\Pi_\delta(d\vartheta)
\succeq
[\overline I_b(\Pi_\delta)+J(\Pi_\delta)]^{-1}_{\beta\beta}.
\label{eq:app-van-trees-matrix}
\end{equation}
Taking traces converts the left side to Bayes fine-prediction risk.

Let
\begin{equation}
\bar N_{\mathrm F}(b)
:=
\int\E_\vartheta N_{\mathrm F}(b)\,\Pi_\delta(d\vartheta),
\qquad
\bar N_{\mathrm C}(b)
:=
\int\E_\vartheta N_{\mathrm C}(b)\,\Pi_\delta(d\vartheta).
\end{equation}
If a policy leaves usable budget idle, adding hypothetical observations can only increase information and decrease the inverse-information trace, so it is enough for the lower bound to complete the design up to a bounded residue $r_b\le\max\{c_{\mathrm F},c_{\mathrm C}\}$. Replacing $b-r_b$ by $b$ changes each $1/b$ leading term by $O(b^{-2})$, which is absorbed below. Let
$\bar\eta_b=c_{\mathrm C}\bar N_{\mathrm C}(b)/b$. At the center $\vartheta_\star$, \Cref{eq:app-beta-schur} gives
\begin{equation}
\operatorname{tr}\left[
\Sigma_{\beta,\star}
(\bar N_{\mathrm F}(b),\bar N_{\mathrm C}(b))^{-1}
\right]
=
\frac{c_{\mathrm F}\sigma_{\mathrm F}^2d}{b}
\psi_{\mathrm U}(\bar\eta_b,\lambda)
+O(b^{-2})
\ge
\frac{\Phi_{\mathrm U}^\star}{b}-O(b^{-2}).
\label{eq:app-center-lower}
\end{equation}

It remains to pass from center information to the prior-averaged information uniformly over policies.

\paragraph{Uniform center-information comparison.}
For $\vartheta\in\mathcal U_\delta$, write
$n_a(\vartheta)=\E_\vartheta N_a(b)$ and let $\mathcal J_{\mathrm C}(\vartheta)$ denote the one-coarse-observation joint information matrix in the local coordinate $(\beta,v)$. The fine information is parameter independent, while
\begin{equation}
\overline I_b(\Pi_\delta)
=
\frac{\bar N_{\mathrm F}(b)}{\sigma_{\mathrm F}^2}
\begin{pmatrix}I_{dK}&0\\0&0\end{pmatrix}
+
\int n_{\mathrm C}(\vartheta)\,
\mathcal J_{\mathrm C}(\vartheta)\,\Pi_\delta(d\vartheta).
\label{eq:app-prior-avg-info}
\end{equation}
If $\bar N_{\mathrm C}(b)=0$, set $\bar\Sigma_\beta=(\bar N_{\mathrm F}/\sigma_{\mathrm F}^2)I_{dK}$; then the comparison below holds with $\gamma=0$. Otherwise normalize the second term by $\bar N_{\mathrm C}(b)$. Because every parameter in the prior support lies within $O(\delta)$ of $\vartheta_\star$, the maps
$w\mapsto ww^\top$, $\Theta\mapsto\Theta Q$, and
$\Theta\mapsto P_{\Theta Q}$ are uniformly Lipschitz, and the normalized coarse-information matrix differs from $\mathcal J_{\mathrm C}(\vartheta_\star)$ by at most $C\delta$ in operator norm. For $\delta$ small enough, its nuisance block remains uniformly positive definite by \Cref{ass:identifiability}. The Schur-complement map is Lipschitz on this compact set, hence the efficient $\beta$ information $\bar\Sigma_\beta$ of the prior-averaged likelihood satisfies
\begin{equation}
\left\|
\bar\Sigma_\beta
-
\Sigma_{\beta,\star}
(\bar N_{\mathrm F}(b),\bar N_{\mathrm C}(b))
\right\|_{\mathrm{op}}
\le
C\delta\gamma,
\qquad
\gamma:=\frac{\bar N_{\mathrm C}(b)}{\sigma_{\mathrm C}^2}.
\label{eq:app-uniform-center-comparison}
\end{equation}

If $\bar\eta_b<1/K$, then
\begin{equation}
\alpha:=\frac{\bar N_{\mathrm F}(b)}{\sigma_{\mathrm F}^2}
\ge
\frac{(1-1/K)b}{c_{\mathrm F}\sigma_{\mathrm F}^2}+O(1)
\asymp b.
\end{equation}
Every $\beta$ Schur complement is bounded below by $\alpha I$. \Eqref{eq:app-uniform-center-comparison} changes the efficient $\beta$ information by at most $O(\delta b)$, while the fixed prior information $J(\Pi_\delta)$ contributes only $O_\delta(1)$. The resolvent identity therefore gives
\begin{equation}
\operatorname{tr}[\overline I_b(\Pi_\delta)+J(\Pi_\delta)]^{-1}_{\beta\beta}
\ge
\operatorname{tr}\left[\Sigma_{\beta,\star}(\bar N_{\mathrm F},\bar N_{\mathrm C})^{-1}\right]
-\frac{C\delta}{b}-\frac{C_\delta}{b^2}.
\label{eq:app-prior-small-coarse}
\end{equation}
Combining with \Cref{eq:app-center-lower} gives the desired lower bound in this case.

If $\bar\eta_b\ge1/K$, use the fixed center fine-only subspace
$V_{w_\star}=\{H:Hw_\star=0\}$ of dimension $d(K-1)$. For a positive-definite operator, the trace of the full inverse is at least the trace of the inverse of its compression to any fixed subspace. By \Cref{eq:app-uniform-center-comparison}, coarse information contributes at most $C\delta\gamma$ on this subspace. Adding $J(\Pi_\delta)$ can increase the compressed information by at most its fixed operator norm $C_\delta$. Hence
\begin{align}
\operatorname{tr}[\overline I_b(\Pi_\delta)+J(\Pi_\delta)]^{-1}_{\beta\beta}
&\ge
\frac{d(K-1)}{\alpha+C\delta\gamma+C_\delta}\\
&\ge
\frac{dKc_{\mathrm F}\sigma_{\mathrm F}^2}{b}-\frac{C\delta}{b}-\frac{C_\delta}{b^2}\\
&\ge
\frac{\Phi_{\mathrm U}^\star-C\delta}{b}-\frac{C_\delta}{b^2},
\end{align}
where the last step uses
$\Phi_{\mathrm U}^\star\le dKc_{\mathrm F}\sigma_{\mathrm F}^2$. Because $\bar\eta_b\ge1/K$, for sufficiently small fixed $\delta$ we have
$\alpha+C\delta\gamma\le[(K-1)/(Kc_{\mathrm F}\sigma_{\mathrm F}^2)+C\delta]b+O(1)$; expanding its reciprocal gives the second line above. The constants in the final $O_\delta(b^{-2})$ term may depend on $\delta$.

Thus both cases yield
\begin{equation}
\operatorname{tr}
[\overline I_b(\Pi_\delta)+J(\Pi_\delta)]^{-1}_{\beta\beta}
\ge
\frac{\Phi_{\mathrm U}^\star-C\delta}{b}
-
\frac{C_\delta}{b^2}.
\end{equation}
Together with \Cref{eq:app-van-trees-matrix}, this is \Cref{eq:adaptive-information-lower}, with $o_\delta(1)=C\delta$. Summation gives the lower half of \Cref{thm:online-main}.

\end{document}